\documentclass[sigconf]{aamas} 

\usepackage{balance} 

\usepackage{times}  
\usepackage{helvet}  
\usepackage{courier}  

\usepackage{newtxtext,newtxmath}
\usepackage{scalerel}
\usepackage{stmaryrd}
\usepackage{dsfont}
\usepackage{xfrac}
\usepackage{mathrsfs}

\usepackage[utf8]{inputenc}
\usepackage{algorithm}
\usepackage{algcompatible}
\usepackage[noend]{algpseudocode}
\usepackage{algpseudocode}
\usepackage{multicol}

\usepackage{booktabs}
\usepackage{arydshln}
\usepackage{here}

\usepackage{xcolor}
\usepackage{framed}
\usepackage[subrefformat=parens]{subcaption}

\usepackage{filecontents}
\usepackage{multirow}
\usepackage{amsmath}

\usepackage{xr}
\DeclareMathOperator*{\concat}{\scalerel*{\Vert}{\sum}}

\newcommand{\argmin}{\mathop{\rm arg~min}}

\newtheorem{theorem}{Theorem}
\newtheorem{proposition}{Proposition}
\newtheorem{lemma}{Lemma}
\newtheorem{definition}{Definition}

\newcommand{\hquad}{\hspace{0.5em}} 

\makeatletter
\newcommand{\figcaption}[1]{\def\@captype{figure}\caption{#1}}
\newcommand{\tblcaption}[1]{\def\@captype{table}\caption{#1}}
\newcommand{\tabenv}{\def\@captype{table}}
\makeatother

\usepackage{algorithm,algpseudocode}
\makeatletter
\newcommand{\algmargin}{\the\ALG@thistlm}
\makeatother
\newlength{\whilewidth}
\algdef{SE}[parWHILE]{parWhile}{EndparWhile}[1]
  {\parbox[t]{\dimexpr\linewidth-\algmargin}{%
     \hangindent\whilewidth\strut\algorithmicwhile\ #1\ \algorithmicdo\strut}}{\algorithmicend\ \algorithmicwhile}%
\algnewcommand{\parState}[1]{\State%
  \parbox[t]{\dimexpr\linewidth-\algmargin}{\strut #1\strut}}

\algdef{SE}[parIF]{parIf}{EndparIf}[1]
  {\parbox[t]{\dimexpr\linewidth-\algmargin}{%
     \hangindent\ifwidth\strut\algorithmicif\ #1\ \algorithmicthen\strut}}%

\setcopyright{ifaamas}
\acmConference[AAMAS '24]{Proc.\@ of the 23rd International Conference
on Autonomous Agents and Multiagent Systems (AAMAS 2024)}{May 6 -- 10, 2024}
{Auckland, New Zealand}{N.~Alechina, V.~Dignum, M.~Dastani, J.S.~Sichman (eds.)}
\copyrightyear{2024}
\acmYear{2024}
\acmDOI{}
\acmPrice{}

\acmSubmissionID{33}

\title[On the Transit Obfuscation Problem]{On the Transit Obfuscation Problem}

\author{Hideaki Takahashi}
\affiliation{
  \institution{The University of Tokyo}
  \city{Tokyo}
  \country{Japan}}
\email{takahashi-hideaki567@g.ecc.u-tokyo.ac.jp}

\author{Alex Fukunaga}
\affiliation{
  \institution{The University of Tokyo}
  \city{Tokyo}
  \country{Japan}}
\email{fukunaga@idea.c.u-tokyo.ac.jp}

\begin{abstract}
Concealing an intermediate point on a route or visible from a route is an important goal in some transportation and surveillance scenarios. This paper studies the Transit Obfuscation Problem, the problem of traveling from some start location to an end location while "covering" a specific transit point that needs to be concealed from adversaries. We propose the notion of transit anonymity, a quantitative guarantee of the anonymity of a specific transit point, even with a powerful adversary with full knowledge of the path planning algorithm.
We propose and evaluate planning/search algorithms that satisfy this anonymity criterion. 
\end{abstract}


\newcommand{\BibTeX}{\rm B\kern-.05em{\sc i\kern-.025em b}\kern-.08em\TeX}

\begin{document}


\pagestyle{fancy}
\fancyhead{}


\maketitle 


\section{Introduction}

In applications such as sensitive cargo transportation or surveillance, it is sometimes necessary to route an agent from a start point to a goal while concealing the location of a {\it transit point}, which is either on the route or visible from the route, from adversaries. 
For example, in a cargo transport application, if a depot or drop location is located somewhere on the route, it is essential to prevent potential adversaries from deducing the transit point location to minimize the risk of theft or interception. In a surveillance application, it is important to be able to conceal which specific location a surveillance agent is targeting on its route.

Obfuscating an agent's true intention has been previously studied in various fields, including path planning, robotics, and game theory~\cite{chakraborti2019explicability}. Previous work has primarily focused on the Goal Obfuscation Problem, which aims to prevent observers from deducing the agent's actual goal, and has numerous applications,  e.g., creating realistic non-player characters (NPCs) capable of deceiving humans~\cite{dias2013great}, or ensuring secure escorting of a VIP to a hidden location~\cite{keren2016privacy}.

Concealing non-goal locations is also important for customer privacy protection in real-world situations. For example, \citet{9851770} considers a scenario where an UAV delivers packages to a private location and returns to the starting point. \citet{chen2012differentially} notes that some public transportation systems track the stations where each customer boards and leaves, and those systems potentially reveal the location of homes or workplaces, which are the transit points on the round-trip paths. 
However, \citet{9851770} is limited to path planning in two-dimensional coordinates with the constraint that the start and goal locations are the same, and \citet{chen2012differentially} focuses on not path planning but anonymizing the collected sequential data.

In this paper, we study the \textbf{Transit Obfuscation Problem},  where given a graph, start location, end location, a target transit point, and a visibility function for an agent, the task is to generate a route from the start to the end location such that the agent's visibility function covers the target, but the target location is concealed from adversaries. 
We assume a strong adversary that has full knowledge of the agent's path, as well as full knowledge of the domain as well the internal decision-making process of the agent (i.e., the adversary has full access to the agent's code). 

We introduce the notion of $(k, \ell, m)$-Anonymity, which quantifies the level of concealment achieved by a path planner. If a path planner satisfies $(k, \ell, m)$-Anonymity, there exist at least $k$ transit points resulting in the same path up to the first $m$ steps, while the deviation among these candidate points is $\geq \ell$. Thus, $(k, \ell, \infty)$-Anonymity is a guarantee that even with full knowledge of the agent's path and code, an adversary can not distinguish the true target transit point among $k$ possible candidates which are at least $\ell$ apart from each other.
We analyze some theoretical properties of  $(k, \ell, m)$-Anonymity and propose a {\it graph partitioning-based} approach to generating paths that guarantee $(k, \ell, m)$-Anonymity. 

The rest of the paper is structured as follows.
First, we define the Transit Obfuscation Problem (TOP) with respect to a path-planning problem with a transit point and visibility constraints (Section \ref{sec:TOP})
Next, in Section \ref{sec:TA}, we define ($k,\ell,m)$-Anonymity for the TOP and analyze its theoretical properties. We also define some metrics for evaluating the tradeoffs between privacy and path costs for the TOP.
Then, in Section \ref{sec:partitioning}, we propose a partitioning-based search algorithm for the TOP which guarantees anonymity even when when the adversary knows the complete path, i.e., $(k,\ell,\infty)$-Anonymity.
Section \ref{sec:m-bounded} proposes algorithms which guarantees anonymity for up to $m < \infty$ steps. In Section \ref{sec:evaluation}, we experimentally evaluate our search algorithms on some standard  benchmark game map instances.
Section \ref{sec:related-work} discusses related work.
Section \ref{sec:conclusion} concludes with a discussion and directions for future work. Our code is available at: \url{https://github.com/Koukyosyumei/TOP}.

\section{Transit Obfuscation Problem} 
\label{sec:TOP}

A \textbf{path-planning domain with visibility constraints} is denoted by a triple $\mathcal{D} = \langle \mathcal{N}, \mathcal{E}, \mathscr{T}, c, v \rangle$, where

\begin{itemize}
    \item $\mathcal{N}$ is a non-empty set of nodes;
    \item $\mathscr{T} \subseteq \mathcal{N}$ is a set of transit candidates;
    \item $\mathcal{E} \subseteq \mathcal{N} \times \mathcal{N}$ is a set of edges between nodes;
    \item $c: \mathcal{E} \rightarrow R^{+}_{0}$ is a function that returns the non-negative cost of an edge between two nodes.
    \item $v: \mathcal{N} \rightarrow 2^{\mathcal{N}}$ is a visibility function that returns the set of visible nodes from a given node.
\end{itemize} 

 The cost of the shortest path (a.k.a minimum cost path) between node $a \in \mathcal{N}$ and node $b \in \mathcal{N}$ is denoted by $d(a, b)$. For simplicity, we assume that $E$ does not contain self-loop edges, i.e., $d(a, b) = \infty$ if $a = b$. A \textbf{path} $\pi$ in a path-planning domain $\mathcal{D}$ is a sequence of nodes $\pi = n_{1}, n_{2}, ..., n_{|\pi|}$ such that $\forall{i \in \{1, ..., |\pi| -1 \}} \hquad (n_{i}, n_{i+1}) \in \mathcal{E}$, where $|\pi|$ represents the length of $\pi$. We also denote the subsequence of $\pi$ till $m$-th node as $\pi|_{m}$, i.e., $\pi|_{m} = n_{1}, n_{2}, ..., n_{m}$. For convenience, we assume that $\pi|_{m} = \pi$ if $m > |\pi|$. The binary operator $\circ$ represents the concatenation of two paths. Specifically, when $\pi_a = a_1, a_2, ..., a_{|\pi_a|}$, $\pi_b = b_1, b_2, ..., b_{|\pi_b|}$, and $a_{|\pi_a|} = b_1$, we have that $\pi = \pi_a \circ \pi_b = a_1, a_2, ..., a_{|\pi_a|}, b_2, ..., b_{|\pi_b|}$. We also introduce $\concat$ notation, where $\concat^{x}_{i=1} \pi_{i} = \pi_{1} \circ \pi_{2} \circ ... \circ \pi_{x}$. The cost of $\pi$ is the sum of the costs of each edge in $\pi$, given by $cost(\pi) = \sum^{|\pi|}_{i = 2} c(\pi_{i-1}, \pi_{i})$, where $\pi_{i}$ is the $i$-th node in $\pi$. We say that $\pi$ \textit{covers} node $n \in \mathcal{N}$ if there exists an index $i$ such that $n \in v(\pi_{i})$. 

A \textbf{path-planning problem with visibility constraints and a transit point (PPVT)} is represented by a tuple $\langle\mathcal{D}, s, g, t \rangle$, where 
 $\mathcal{D} = \langle \mathcal{N}, \mathscr{T}, \mathcal{E}, c, v \rangle$ is a domain, 
 $s \in \mathcal{N}$ is the start node, 
 $g \in \mathcal{N}$ is the goal node,
 $t \in \mathscr{T}$ is the transit point that must be covered.
 The solution to a PPVT is a path $\pi$ such that $\pi_{1} = s$, $\pi_{|\pi|} = g$, and $\exists{i \in \{1, 2, ..., |\pi|\}, \hquad t \in v(\pi_{i})}$. 

A \textbf{path planner} $\mathcal{A}$ takes as input a PPVT and 
returns a feasible path $\pi$ for that problem, i.e., $\mathcal{A}(\langle \mathcal{D}, s, g, t \rangle) = \pi$. 
$\mathcal{A}$ returns Failure when it cannot find a feasible path. For convenience,  we define Failure such that it is not equal to itself, i.e., Failure $\neq$ Failure.


We assume that there is an adversary who seeks to deduce the actual transit point $t \in \mathscr{T}$ by observing the trajectory of the agent. 

A \textbf{Transit Obfuscation Problem} (TOP) is a tuple $\langle \mathcal{D}, s, g, \mathscr{O} \rangle$, where $\mathscr{O}$ describes what the adversary can observe. 

We make the following assumptions about the abilities of the adversarial observer (similar to the set of assumptions by ~\cite{kulkarni2018resource}):

\begin{itemize}
    
    \item \textbf{Complete Knowledge about the Domain and Transit Candidates}: The adversary has complete knowledge about the domain $\mathcal{D}$.
    \item \textbf{Full Access to the Planner}: The adversary has full access to and thoroughly understands the agent's planning algorithm.
    \item \textbf{Independence of Inputs}: The adversary can execute the agent's planner with arbitrary input tuples at any time.
    \item \textbf{Observability of Path}: The adversary can immediately observe the path executed by the agent so far.
    \item \textbf{Semi-Honest Adversary}: The adversary is passive, and it does not disturb the action of the agent or gain any additional information beyond what has been specified above.
\end{itemize}

These are challenging assumptions when trying to conceal the transit point, as the adversary has full information about the mechanism of the path planning algorithm, as well as the ability to rerun/simulate the algorithm many times in order to gain information that might reveal the transit point. When $t = g$, and $v$ is the identity function (the only node visible from a node is itself), this special case of the TOP is similar to the Goal Obfuscation Problem~\cite{kulkarni2018resource, kulkarni2019unified}. 
Unlike the Goal Obfuscation Problem, where the final node of the path always reveals the actual goal,
in the TOP, when $t \neq g$, it is possible to have a  path where an adversary cannot deduce the actual transit point even after observing the entire trajectory.
For example, in Fig.~\ref{fig:intuition}, an agent travels from $s$ to $g$ while covering one of $\mathscr{T} = \{t_1, t_2, t_3, t_4\}$, where black cells are obstacles. Then, if the agent can see nodes within a radius of one, any feasible path covers all of $\mathscr{T}$ so that an observer cannot infer the true transit point.

\begin{figure}
    \centering
    \includegraphics[width=\linewidth]{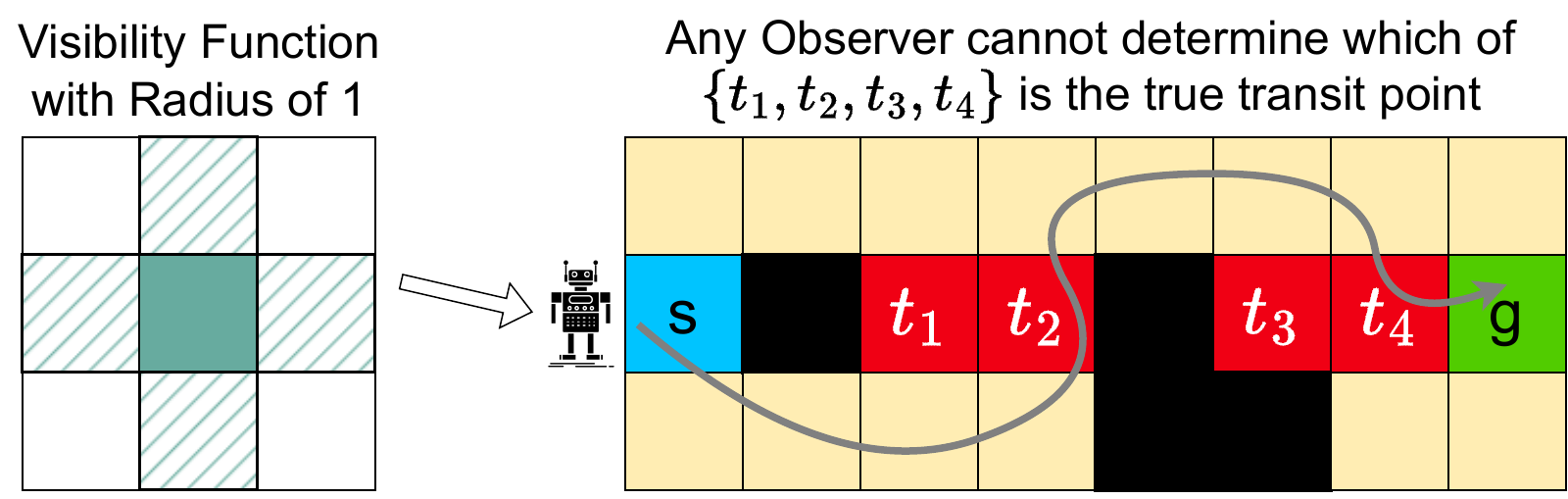}
    \caption{Let nodes within a radius of one be visible. Then, it is not possible for an observer to determine which of ${t_1, t_2, t_3, t_4}$ is the true transit point.}
    \label{fig:intuition}
\end{figure}

\section{$(k, \ell, m)$-Anonymity}
\label{sec:TA}

Next, we formally define the conditions when a path is anonymized for a transit point, what inputs are anonymizable, and what planners can achieve anonymization.

\subsection{Definitions of $(k, \ell, m)$-Anonymity}

In order for a transit point $t$ to remain private up to the first $m$ steps, even if the adversary has the capabilities enumerated above, it must not be possible to uniquely identify $t$ by executing the planner (possibly many times) and observing the output(s) as well as the internal state of the planner.

This is possible if the route output by a planner is indistinguishable for multiple transit candidate points, including the actual transit point $t$. For example, if the set of transit candidates  
$\mathscr{T} = \{t_1, t_2\}$, and 
$\mathcal{A}(\langle \mathcal{D}, s, g, t_1 \rangle) = 
\mathcal{A}(\langle \mathcal{D}, s, g, t_2 \rangle) = 
\pi_{1,2}$, it is indeterminate which of $t_1$ or $t_2$ is the true transit point $t$.

In addition, it is often desirable for the transit candidates to be spread out. For example, if all the transit candidates are close to each other, the adversary may be able to cost-effectively block access to all transit candidates (preventing the agent from covering the transit point). 
Therefore, it is desirable to be able to disperse  the indistinguishable transit candidates in the search space.  

Based on this idea,  we first define $(k, \ell, m)$-Anonymized Paths.

\begin{definition}[$(k, \ell, m)$-Anonymized Path]
\label{def:transit-anonymized-path}
Let $k \in \mathbb Z_{> 0}$, $\ell \in \mathbb R_{\geq 0} $ and $m \in \mathbb Z_{> 0}$. We say that $\mathcal{A}(\langle \mathcal{D}, s, g, t \rangle)$, a path planned by $\mathcal{A}$ from $s$ to $g$ covering $t$, where $s \neq t$ and $g \neq t$,  is a $(k, \ell, m)$-Anonymized Path with respect to $t$ if there exists a set $T \subseteq \{ t' | t' \in \mathscr{T} \text{ and } \mathcal{A}(\langle \mathcal{D}, s, g, t \rangle)|_{m} = \mathcal{A}(\langle \mathcal{D}, s, g, t' \rangle)|_{m} \}$ satisfying $|T| \geq k$ and $\min_{(i, j) \in T \times T } d(i, j) \geq \ell$.
\end{definition}

\noindent In other words, a path is  $(k, \ell, m)$-Anonymized when there are at least $k$ transit candidates that result in the same path up to the first $m$ nodes from $s$ to $g$ covering $t$, and the distance between any pair of nodes within that set is equal to or greater than $\ell$. When $m = \infty$, the adversary cannot determine which of those transit candidates is the true $t$ even after observing the entire path. 


Second, we define a $(k, \ell, m)$-Anonymizable Tuple. 

\begin{definition}[$(k, \ell, m)$-Anonymizable Tuple]
Let $k \in \mathbb Z_{> 0}$ and $\ell \in \mathbb R_{\geq 0} $. Given a domain $\mathcal{D}$, we say that the tuple $\langle \mathcal{D}, s, g, t \rangle$, where $s \in \mathcal{N}$, $g \in \mathcal{N}$, $t \in \mathscr{T}$, $s \neq t$, and $g \neq t$, is a $(k, \ell, m)$-Anonymizable Tuple if there exists a path planner $\mathcal{A}$ such that there exists a set $T \subseteq \{ t' | t' \in \mathscr{T} \text{ and } \mathcal{A}(\langle \mathcal{D}, s, g, t \rangle)|_{m} = \mathcal{A}(\langle \mathcal{D}, s, g, t' \rangle)|_{m} \}$ satisfying $|T| \geq k$ and $\min_{(i, j) \in T \times T} d(i, j) \geq \ell$.
\end{definition}

\noindent The input tuple is $(k, \ell, m)$-Anonymizable when at least one planner can output a $(k, \ell, m)$-Anonymized Path for this input. Since $k$ is a positive integer, a tuple $\langle \mathcal{D}, s, g, t \rangle$ is not $(k, \ell, m)$-Anonymizable Tuple for any $k$ if there exists no path from $s$ to $g$ covering $t$. 

Based on the above definitions, we define $(k, \ell, m)$-Anonymity of a path planner $\mathcal{A}$ as follows.

\begin{definition}[$(k, \ell, m)$-Anonymity]
A path planner $\mathcal{A}$ satisfies $(k, \ell, m)$-Anonymity for a domain $\mathcal{D}$ if $\mathcal{A}$ returns a $(k, \ell, m)$-Anonymized Path for all $(k, \ell, m)$-Anonymizable Tuples in $\mathcal{D}$.
\end{definition}

\noindent A path planner with $(k, \ell, m)$-Anonymity is guaranteed to return anonymized paths for all anonymizable tuples in $\mathcal{D}$.


We also consider the anonymity of the planner for fixed source and goal locations and define $(k, \ell, m, \delta)$-Local Anonymity as follows:

\begin{definition}[$(k, \ell, m, \delta)$-Local Anonymity]
Let $x$ be the number of $(k, \ell, m)$-Anonymizable Tuples in the domain $\mathcal{D}$ with a fixed source $s$ and goal $g$. We say that $\mathcal{A}$ satisfies $(k, \ell, m, \delta)$-Local Anonymity for 
$\langle \mathcal{D}, s, g \rangle$
if $\mathcal{A}$ returns $(k, \ell, m)$-Anonymized Paths for $\delta x$ or more $(k, \ell, m)$-Anonymizable Tuples in $\mathcal{D}$ with $s$ and $g$.
\end{definition}

\noindent A planner $\mathcal{A}$ satisfying $(k, \ell, m)$-Anonymity in $\mathcal{D}$ satisfies $(k, \ell, m, 1)$-Local Anonymity for any combination of a start $s$ and a goal $g$. 

\subsection{Properties of $(k, \ell, m)$-Anonymity}

We have identified several important properties about $(k, \ell, m)$ - Anonymity: 
equivalence conditions for $(k, \ell, m)$ - Anonymizable Tuples and Anonymized Paths, path-extensibility, and existence guarantee of a planner achieving $(k, \ell, m)$-Anonymity. All omitted proofs, as well as some additional properties can be found in Supp.~\ref{supp:proof:sec3}~\cite{takahashi2024top-supp}.

First, the following proposition indicates the necessary and sufficient conditions for a path to be a $(k, \ell, \infty)$-Anonymized Path.

\begin{proposition}[3C Condition for Output Path]
\label{thm:pathcoverage-path}
A path $\pi$ = $\mathcal{A}(\langle \mathcal{D}, s, g, t \rangle)$ is $(k, \ell, \infty)$-Anonymized Path iff there exists a set of nodes $T \subseteq \mathscr{T}$ satisfying all of the following:
    \begin{enumerate}
        \item Cardinality: $|T| \geq k$
        \item Cost: $\min_{(i, j ) \in T \times T } d(i, j) \geq \ell$ 
        \item Coverage: $\pi$ covers all nodes in $T$, and $\mathcal{A}$ returns $\pi$ whenever the transit point belongs to $T$.
    \end{enumerate}
\end{proposition} 

The similar necessary and sufficient conditions for an input tuple to be $(k, \ell, \infty)$-Anonymizable are in Supp.~\ref{supp:additional-properties}~\cite{takahashi2024top-supp}.

Next, the coverage condition of Prop.~\ref{thm:pathcoverage-path} leads to the computational complexity of planning a $(k, \ell, \infty)$-Anonymized Path.

\begin{theorem}[ Complexity]
\label{thm:np}
Finding a $(k, \ell, m)$-Anonymized Path for the given tuple is NP-Hard.
\end{theorem}

\begin{proof}[Proof of Theorem~\ref{thm:np}]
Finding a path that covers all nodes in the given set of nodes is a generalization of WRP, which is NP-Hard~\cite{seiref2020solving} (see Sec.~\ref{subsec:wrpt})
\end{proof}

If the domain consists of an undirected graph, we can ensure the existence of a path planner that satisfies $(k, \ell, m)$-Anonymity.

\begin{theorem}[Existence of a Satisfying Path Planner]
\label{thm:existence}
If every edge in the domain $\mathcal{D}$ is undirected, meaning that $\forall (i, j) \in \mathcal{N} \times \mathcal{N}, \hquad (i, j) \in \mathcal{E} \Rightarrow (j, i) \in \mathcal{E}$, there exists a path planner $\mathcal{A}$ that satisfies $(k, \ell, m)$-Anonymity for any given $k$, $\ell$ and $m$.
\end{theorem}

To prove this, we use the Lemma below which states that extending a $(k, \ell, m)$-Anonymized Path preserves the same level of anonymity.

\begin{lemma}[Path-Extension]
\label{lem:path-extension}
Let $\pi = A(\langle \mathcal{D}, s, g, t \rangle)$ be a $(k, \ell, m)$-Anonymized Path with respect to $t$, $\pi_{s' \to s}$ be a path, constructed independently of $t$, from $s'$ to $s$, and $\pi_{g \to g'}$ be an arbitrary path from $g$ to $g'$. Then, $\pi_{s' \to s} \circ \pi \circ \pi_{g \to g'}$ is also a $(k, \ell, m)$-Anonymized Path with respect to $t$.
\end{lemma}

\begin{proof}[Proof of Theorem~\ref{thm:existence}]

Let $\hat{T}_{s, g} = \{t \mid t \in \mathscr{T}, \langle D, s, g, t \rangle$ is $(k, \ell, m)$ -Anonymizable Tuple$\} = \{t_1, t_2, ..., t_x\}$, and $\hat{\pi}^{t_i}_{s \to g}$ be the $(k, \ell, m)$-Anonymized Path for $\langle D, s, g, t_i \rangle$. If $|\hat{T}_{s, g}| \geq 1$ and all edges in $\mathcal{D}$ are undirected, there exists $\pi_{g \to s}$, a path from $g$ to $s$. By Lemma~\ref{lem:path-extension}, we have that $\hat{\pi}_{s \to g} = (\concat^{x-1}_{i=1} \hat{\pi}^{t_i}_{s \to g} \circ \pi_{g \to s}) \circ \hat{\pi}^{t_x}_{s \to g}$ is a $(k, \ell, m)$-Anonymized Path with respect to all transit nodes in $\hat{T}_{s, g}$.

Now, consider a path planner $\mathcal{A}$ such that $\mathcal{A}(\langle D, s, g, t \rangle) = \hat{\pi}_{s \to g}$ if $|\hat{T}_{s, g}| \geq 1$ and $\mathcal{A}(\langle D, s, g, t \rangle) =$ Failure otherwise. It is evident that $\mathcal{A}$ satisfies $(k, \ell, m)$-Anonymity:
\end{proof}

Although Theorem.~\ref{thm:existence} guarantees that there exists a path planner with $(k, \ell, m)$-Anonymity for any undirected graph, such a guarantee is not possible when the edges are directed. 
A counterexample is shown in the left-side of Fig.~\ref{fig:failure}, where there are two possible paths from $s$ to $g$; $\pi_a = (s,t_1,t_2,t_3,g)$ or $\pi_b = (s,t_1,t_4,t_5,g)$. Thus, if we assume that the costs of edges are all one, all input tuples are $(3, 1, \infty)$-Anonymizable Tuples. Let $\pi_{i} = \mathcal{A}(\langle \mathcal{D}, s, g, t_i \rangle)$. 
Clearly, any path planner $\mathcal{A}$ should satisfy $\pi_a = \pi_2 = \pi_3$ and $\pi_b = \pi_4 = \pi_5$. Then, if $\pi_a = \pi_1$, $\pi_a$ is a $(3, 1, \infty)$-Anonymized Path but $\pi_b$ is not a $(3, 1, \infty)$-Anonymized Path. Likewise, if $\pi_b = \pi_1$, $\pi_b$ is a $(3, 1, \infty)$-Anonymized Path but $\pi_a$ is not a $(3, 1, \infty)$-Anonymized Path. On the other hand, if all edges are undirected, we can construct a $(3, 1, \infty)$-Anonymized Path from $\pi_{a}$ and $\pi_{b}$ by concatenating them.

\begin{figure}
    \centering
    \includegraphics[height=2.0cm]{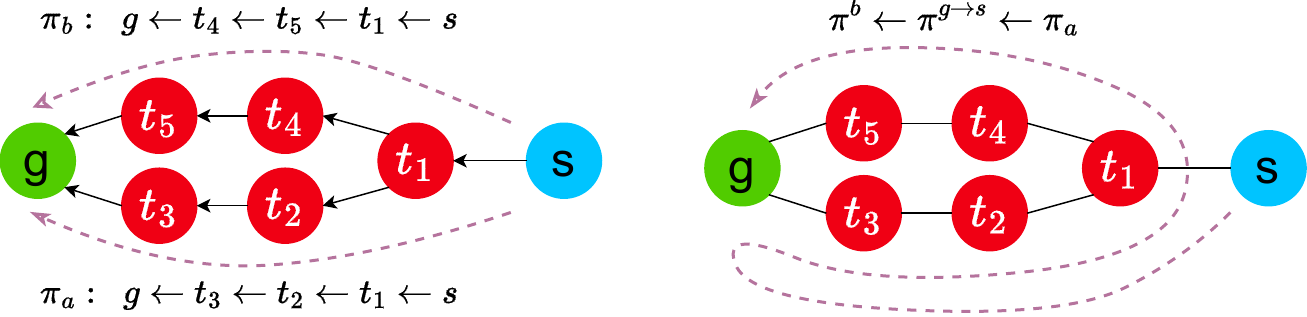}
    \caption{Directed vs. Undirected: While all tuples are $(3, 1, \infty)$-Anonymizable Tuples, planning $(3, 1, \infty)$-Anonymized Path for every transit point is impossible in the directed (left) case.}
    \label{fig:failure}
\end{figure}

\section{PbP: Partitioning-based Planner for $(k,\ell,\infty)$-Anonymity}
\label{sec:partitioning}

We now propose an algorithm for $(k,\ell,\infty)$-Anonymity. First, note that {\it if we disregard path cost}, then in principle, a relatively straightforward approach to achieve $(k,\ell,\infty)$-Anonymity would be to return some path which covers all transit candidates.
However, a practical algorithm for Transit Obfuscation needs to effectively trade off the privacy objective vs. path costs, on average, overall $(s,g,t)$ tuples of interest.
We first propose objectives that express this tradeoff and then propose a partitioning-based algorithm which seeks a solution which optimizes this objective. Proofs are in Supp.~\ref{supp:proof:sec4}~\cite{takahashi2024top-supp}.

\subsection{Objectives}

We define two metrics, Anonymized Path Ratio (APR) and Mean Anonymization Cost (MAC), to evaluate the performance of planners that satisfy $(k,\ell,m)$-Anonymity.
Let $\mathscr{T}_{+}$ be the set $\{t | t \in \mathscr{T}, \mathcal{A}(\langle \mathcal{D}, s, g, t\rangle) \text{ is a $(k, \ell, m)$-Anonymized Path}\}$ for the fixed $\mathcal{D}$, $s$, and $g$. APR is defined as follows:

\begin{definition}[Anonymized Path Ratio (APR)]

\begin{equation*}
    \hbox{APR}(\mathcal{A}, D, s, g) = \frac{|\mathscr{T}_{+}|}{\hbox{\#Coverable Transit Nodes}}
\end{equation*}

\end{definition} \noindent where \textit{\#Coverable Transit Nodes} denotes the number of transit nodes within $\mathscr{T}$ such that there exists a path from $s$ to $g$ covering $t \in \mathscr{T}$. APR is equivalent to the lower bound of $\delta$, and a larger APR is desirable. 

Next, inspired by deception cost~\cite{price2023domain}, a metric for goal obfuscation, we define the MAC metric: 

\begin{definition}[Mean Anonymization Cost (MAC)]
    \begin{equation*}
        \hbox{MAC}(\mathcal{A}, D, s, g) = \sum_{t \in \mathscr{T}_{+}} \frac{cost(\mathcal{A}(\langle \mathcal{D}, s, g, t \rangle)) - cost(\pi^{t*})}{ |\mathscr{T}_{+}| \hquad cost(\pi^{t*}))}
    \end{equation*}
\end{definition} \noindent where $\pi^{t*}$ is the shortest path from $s$ to $g$ covering $t$. MAC shows the cost of anonymizing the transit points.

\subsection{Pbp: Partitioning-based Planner}

\begin{algorithm}[!th]
\caption{Partitioning-based Planner (Pbp)}
\label{alg:1}
\begin{algorithmic}[1]

\Require Tuple $\langle \mathcal{D}, s, g, t \rangle$ and privacy parameters $(k, \ell)$
\Ensure a path $\pi$ from $s$ to $g$ covering $t$

\State /* Pre-Processing \textbf{independent of} $t$*/
\State Generate a partition of $\mathscr{T}$, $T_{1}, T_{2}, ..., T_{\Phi}$ such that $\bigcup^{\Phi}_{\phi=1} T_{\phi} = \mathscr{T}$, any $T_{\phi}$ except $T_{\Phi}$ meets all conditions of Prop.~\ref{thm:pathcoverage-path}, and any pair are disjoint.
\State /* End of Pre-Processing */

\If{$t$ belongs to $T_{\Phi}$}
\Return Failure
\EndIf

\State $T \leftarrow$ the partition that contains $t$
\State $\pi^{*}_{T} \leftarrow$ shortest path from $s$ to $g$ while covering $T$ \label{alg:1-call-wrpt}
\State 
\Return $\pi^{*}_{T}$
\end{algorithmic}
\end{algorithm}

We now propose the Partitioning-based Planner (Pbp),  a practical algorithm that seeks to achieve $(k, \ell, \infty)$-Anonymity while optimizing the above objectives.

Pbp is composed of two phases: (1) the Partitioning/Pre-Processing phase and (2) the path query phase.
In the Partitioning/Pre-Processing phase, 
the algorithm searches for a partition of all transit candidates, denoted as $T_1, T_2, ..., T_\Phi$, such that (a) each subset except the final $T_\Phi$ satisfies the 3C Conditions described in Prop.~\ref{thm:pathcoverage-path}, ensuring that each subset covers the required nodes to achieve the desired level of anonymity, and (b) the objectives are optimized. The set  $T_\Phi$ consists of transit candidates that the planner cannot anonymize.
This phase only needs to be executed once for each domain.

In the path query phase, given a specific $s$, $g$, and $t$, the algorithm finds the shortest path $\pi$ from the source node $s$ to the goal node $g$ while covering all the nodes assigned to the subset (computed above in the Partitioning phase) that includes the target node $t$. 

By utilizing this approach, Pbp can effectively plan a  $(k, \ell, \infty)$-Anonymized Path, ensuring the privacy requirements are met while efficiently navigating from the source to the destination. If the obtained partition is perfect, Pbp satisfies $(k, \ell, \infty)$-Anonymity.

\begin{theorem}[Completeness]
\label{thm:comp}
Let all edges in $\mathcal{D}$ be undirected. Then, if $\sum^{\Phi-1}_{\phi=1} |T_{\phi}|$ is maximized for any pair of $s$ and $g$, Alg.~\ref{alg:1} satisfies $(k, \ell, \infty)$-Anonymity.   
\end{theorem}

\subsection{WRP With Targets (WRPT)}
\label{subsec:wrpt}

A key building block for the Pbp algorithm is a search algorithm for
finding the minimum cost path $\pi$ from $s$ to $g$, which covers all of the nodes in a set of nodes. We call this the Watchman Route Problem with Targets (WRPT).  The WRPT corresponds to the subproblem solved by the path query phase of Pbp in Alg.~\ref{alg:1}, line~\ref{alg:1-call-wrpt}. The WRPT is also used in the Partitioning phase when evaluating candidate partitionings (Alg.~\ref{alg:mergebb}, line~\ref{merge-bb:call-wrpt}).

The WRPT is a variant of the Watchman Route Problem (WRP) \cite{SkylerAYF22}, 
The differences between the WRP and WRPT are: (1)  WRP does not have a specified goal node, while the WRPT has a specific goal $g$, and (2) the objective of the WRP is to cover all nodes in the graph, while the WRPT seeks to cover some subset $\psi$ of nodes in the graph. Since WRP was shown to be NP-Hard~\cite{SkylerAYF22}, the WRPT is clearly NP-Hard. Recent work has studied heuristic search-based algorithms to solve WRP ~\cite{SkylerAYF22}.

Following~\citet{SkylerAYF22}, we use an A* search for the WRPT. We define a state for the search as a tuple $\langle n, \mathcal{U} \rangle$, where $n \in \mathcal{N}$ represents the current location, and $\mathcal{U} \subseteq \mathcal{N}$ represents the set of uncovered nodes. The initial state is $\langle s, \psi \setminus v(s) \rangle$, and the final is $\langle g, \emptyset \rangle$. Expanding a state $\langle n, \mathcal{U} \rangle$ involves moving from $n$ to one of its neighboring nodes $n'$ and updating the set of uncovered nodes $\mathcal{U}$ to $\mathcal{U} \setminus v(n')$. The cost of this expansion equals $c(n, n')$. 

To make the search more efficient, we propose Tunnel Heuristic, which is based on the Singleton Heuristic for the WRP~\cite{SkylerAYF22}. The Tunnel Heuristic value $h_{tunnel}$ is computed as follows:

\begin{align*}
&h_{tunnel}(\langle n, \mathcal{U} \rangle) &= \begin{cases}
    (\max\limits_{u \in \mathcal{U}} \min\limits_{q \in v^{-1}(u)} d(n, q)) \\ + (\min\limits_{u \in \mathcal{U}} \min\limits_{q \in v^{-1}(u)} d(q, g)) &\text{if $\mathcal{U} \neq \emptyset$} \\
    d(n, g) &\text{otherwise}
\end{cases} 
\end{align*}

\noindent where the function $v^{-1}(n): \mathcal{N} \to 2^{\mathcal{N}}$ takes a node and returns the set of nodes from which $n$ is observable. The heuristic $h_{tunnel}$ is admissible since the agent must travel to one of the nodes in $v^{-1}(u)$ to observe an uncovered node $u$ and then proceed to the goal $g$ after covering all nodes in $\mathcal{U}$.

\subsection{Searching for a Partitioning}

Alg.~\ref{alg:1} requires an algorithm that generates a partition of the set of nodes. Let $\Psi_{+} \subseteq \Psi$ be the largest subset of $\Psi$ whose elements all satisfy the conditions of Prop.~\ref{thm:pathcoverage-path}. We denote the sum of the cardinalities of the subsets in $\Psi_{+}$ as $|ap|$, and the \textit{MAC} corresponding to $\Psi_{+}$ as $mac$. Specifically, $mac$ is $\sum_{\psi \in \Psi_{+}} ac(\psi) / |ap|$, where $ac(\psi) = \sum_{k \in \psi} (\pi^{*}_{\psi} - \pi^{k*})/(\pi^{k*})$, where $\pi^{k*}$ denotes the minimal cost path from $s$ to $g$ covering $k$, and $\pi^{*}_{\psi}$ denotes the minimal cost path from $s$ to $g$ while covering all nodes within $\psi$, i.e., the solution to a WRPT which covers $\psi$. 
We seek a partitioning which first prioritizes maximizing $|ap|$,
then minimizes $mac$, i.e., a partitioning which 
anonymizes as many transit nodes as possible while minimizing the average cost of the anonymized paths.

\begin{algorithm}[!th]
\caption{Merge-BB Partitioning}
\label{alg:mergebb}
\begin{algorithmic}[1]

\Require Tuple $\langle \mathcal{D}, s, g, t \rangle$ and privacy parameters $(k, \ell)$
\Ensure The best partition $\Psi^{*}$ of $\mathscr{T}$

\State $|ap|^{*} \leftarrow 0$, $mac^{*} \leftarrow \infty$, $\Psi^{*} = \emptyset$

\Function {Merge\_BB\_Search}{$\Psi$}
    \State $\Psi_{+} \leftarrow \{\psi | \psi \in \Psi$ s.t. $\psi$ satisfies all conditions of Prop.~\ref{thm:pathcoverage-path} $\}$
    \State $|ap| \leftarrow \sum_{\psi \in \Psi_{+}} |\psi|$, $mac \leftarrow \sum_{\psi \in \Psi_{+}} ac(\psi) / |ap|$

    \If {($|ap| > |ap|^{*}$) or ($|ap| = |ap|^{*}$ and $mac < mac^{*}$)}
        \State $|ap|^{*} \leftarrow |ap|$, $mac^{*} \leftarrow mac$, $\Psi^{*} \leftarrow \Psi$
    \EndIf

    \If {($|\Psi| = 1$) or ($|ap| = \sum_{\psi \in \Psi} |\psi|$) or ($|ap|^{*} = \sum_{\psi \in \Psi} |\psi|$ and $mac \geq mac^{*}$)} \label{merge-bb:termination}
    \Return True
    \EndIf


    \For{$(i,j) \in MergeOrder(\Psi)$} \label{merge-bb:mergeorder}
    \If {Prunable($\psi_i,\psi_j$)} \label{merge-bb:prunable}
        Continue
    \EndIf
    
    \parState {$\pi^{*}_{\psi_{i} \cup \psi_{j}} \leftarrow$ shortest path from $s$ to $g$, covering $\psi_{i} \cup \psi_{j}$ }
    \If{$\pi^{*}_{\psi_{i}}$ is not found}
        Continue
    \EndIf
    \State $\Psi' \leftarrow \Psi \setminus \{\psi_{i}, \psi_{j}\} \cup \{\psi_{i} \cup \psi_{j}\}$
    \State Merge\_BB\_Search($\Psi'$)
    \EndFor
\EndFunction

\State
\For{$i \in \mathscr{T}$}
    \If{$\exists u \in v(i)$ $u$ is reachable from $s$ and $\exists u \in v(i)$ $g$ is reachable from $u$}
        \State $\psi_{i} \leftarrow \{i\}$
        \State $\pi^{*}_{\psi_{i}} \leftarrow $ shortest path from $s$ to $g$, covering $t$ \label{merge-bb:call-wrpt}
    \EndIf
\EndFor
\State Merge\_BB\_Search($\{\psi_{1}, \psi_{2}, ...\}$)
\State
\Return $\Psi^{*}_{+} \cup \{\bigcup \Psi^{*} \setminus \Psi^{*}_{+} \}$
\end{algorithmic}
\end{algorithm}

\subsubsection{Merge-based Branch-and-Bound}

One practical partitioning algorithm is Merge-based Branch-and-Bound Partitioning (Merge-BB). It initially assigns each node to its own separate partition. It removes any node without a valid path from $s$ to $g$ while covering that node. This pruning step involves calculating all pair-wise shortest paths on the node set $N$, which can be done efficiently within a reasonable amount of time. The algorithm then 
performs a recursive branch-and-bound search which considers all possible combinations of merges of these partitions.

The termination condition for this recursive search is implemented in Line \ref{merge-bb:termination}: Return True when (1) the partition $\Psi$ contains only one subset, or (2) $|ap|$ reaches the upper bound, or (3) the current best partition $\Psi^*$ has an optimal $|ap|$, and the $mac$ of $\Psi$ is not better than $mac^{*}$. The third termination condition is based on the observation that the cost of the optimal path covering all nodes in the union of $\psi_{i}$ and $\psi_{j}$ is always equal to or greater than both of the costs of the optimal paths covering all nodes in $\psi_{i}$ and $\psi_{j}$.

Alg.~\ref{alg:mergebb} explores all potential merges and returns a partition that maximizes the number of anonymized transit points while minimizing the average cost of anonymized paths.

\begin{proposition}[Optimality]
\label{prop:opt}
Alg.~\ref{alg:1} using Alg~\ref{alg:mergebb} achieves the largest APR and also has the minimum MAC among planners with the largest APR for any combination of $\mathcal{D}$, $s$, and $g$.
\end{proposition}

Since Theorem~\ref{thm:existence} tells that there exists a planner satisfying $(k, \ell, \infty)$-Anonymity for an undirected graph, which means that it can anonymize all $(k, \ell, \infty)$ - Transit Anonymizable Tuples, combining Theorem~\ref{thm:comp} and Prop.~\ref{prop:opt} immediately yields the guarantee that Alg.~\ref{alg:1} with Alg.~\ref{alg:mergebb} satisfies $(k, \ell, \infty)$-Anonymity. 

We also implemented the following enhancements.

\paragraph{Merge Ordering Strategies}

The order in which partitions are merged in Alg.~\ref{alg:mergebb} by the recursive enumeration is determined by a call to the MergeOrder function in line \ref{merge-bb:mergeorder}, which returns the list of all pairs of candidate subsets to merge, sorted according to some merge ordering criterion. One simple strategy is Random, which simply returns a randomly shuffled list of the pairs of partitions. Another ordering strategy, CostAsc, sorts the pairs to be merged in ascending order of a heuristic cost function. We use $max(cost(\pi_{\psi_{i}}), cost(\pi_{\psi_{j}})) \times (|\psi_{i}| + |\psi_{j}|)$ as the cost of a pair $(\psi_{i}, \psi_{j})$, where the first term is the lower bound of the covering path of the merged partition, and the second term is the number of transit candidates assigned to the merged partition. CostAsc helps the planner find a better solution earlier, leading to more upper/lower bound-based pruning.

\paragraph{Pruning Criteria}
To determine whether we need to try merging $(\psi_{i}, \psi_{j})$,  Alg.~\ref{alg:mergebb}, line \ref{merge-bb:prunable}  calls Prunable (Alg.\ref{alg:prunable}). If both $\psi_{i}$ and $\psi_{j}$ already satisfy all the conditions stated in Thm.\ref{thm:pathcoverage-path}, the merge is pruned because it would increase the cost of a path covering all the nodes within the set (Alg.~\ref{alg:prunable}, Line~\ref{prunable:condition1}). Furthermore, suppose we have already found a satisfactory solution for anonymizing all possible tuples. In that case, we can establish an upper bound for the path cost covering the union of $\psi_{i}$ and $\psi_{j}$ to surpass the current best satisfying solution and prune based on this bound (Alg.~\ref{alg:prunable}, Line~\ref{prunable:condition2}). Specifically, we denote $\hat{ac}$ as the upper bound for the cost that $\pi_{\psi_{i} \cup \psi_{j}}$ must meet to improve upon the best $ac^{*}$, and it is calculated as $|ap|^{*} mac^{*} - |ap| mac + ac(\pi_{\psi_i}) + ac(\pi_{\psi_j})$. To estimate the cost of a path that covers all the nodes within the union of $\psi_{i}$ and $\psi_{j}$, we use $\max(cost(\pi^{*}_{\psi_i}), cost(\pi^{*}_{\psi_j}))$, which is the lower-bound of $cost(\pi^{*}_{\pi_i \cup \pi_j})$. Finally, we check the minimum distance between a node in $\psi_{i}$ and $\psi_{j}$. If that distance is less than $\ell$, we prune this merge, as any set containing the union of $\psi_{i}$ and $\psi_{j}$ would also violate this condition (Alg.~\ref{alg:prunable}, Line~\ref{prunable:condition3}).

\begin{algorithm}[!th]
\caption{Prunable}
\label{alg:prunable}
\begin{algorithmic}[1]


\If{Both $\psi_{i}$ and $\psi_{j}$ meets all conditions in Prop.~\ref{thm:pathcoverage-path}}
    \State
    \Return True \label{prunable:condition1}
\EndIf

\If{$\min_{(x, y) \in \psi_{i} \times \psi_{j}} d(x, y) < \ell$} 
    \Return True \label{prunable:condition2}
\EndIf

\If{$|ap|^{*} = \sum_{\psi \in \Psi} |\psi|$}
    \State $\hat{ac} = |ap|^{*} mac^{*} - |ap| mac + ac(\pi_{\psi_i}) + ac(\pi_{\psi_j})$
    \State $\bar{ac} = \sum_{k \in \psi_i \cup \psi_j} \frac{\max(cost(\pi^{*}_{\psi_i}), cost(\pi^{*}_{\psi_j})) - cost(\pi^{k*})}{cost(\pi^{k*})}$
    \If{$\bar{ac} \geq \hat{ac}$}
        \Return True \label{prunable:condition3}
    \EndIf
\EndIf

\State
\Return False

\end{algorithmic}
\end{algorithm}

\section{Planners for $m$-bounded $(k,l,m)$-Anonymity}
\label{sec:m-bounded}
We now propose three planners that output $(k, \ell, m)$-Anonymized Paths with $m < \infty$, which means that the adversary cannot identify which node is the transit point until observing more than $m$ nodes. If $m$ is less than the length of the output path, the adversary might be able to identify the true transit point after obtaining the $m+1$st and later nodes. Proofs for this section can be found in Supp.~\ref{supp:proof:sec5}~\cite{takahashi2024top-supp}.

\paragraph{Random-Walk-based Planner (Rbp)}

The first algorithm is a Random-Walk-based Planner (Rbp), which randomly selects the path's first $m$ nodes. Specifically, the planner selects the $i$-th node ($2\leq i \leq m$) of the path randomly from the neighbors of the $(i-1)$-th node, where $\pi_{1}=s$. Subsequently, the planner guides the agent's movement from $\pi_{m}$ to $t$ and finally to $g$ utilizing the shortest paths available. $\pi|_{m}$ is the same for all transit candidates, and the planner outputs Failure if there is no feasible path. This planner archives $(k, \ell, m)$-Anonymity with finite $m$ for undirected graphs.

\begin{proposition}
\label{prop:rbp}
    If all edges in $\mathcal{E}$ are undirected, Random-Walk-based Planner satisfies $(k, \ell, m)$-Anonymity with $m < \infty$.    
\end{proposition}

\paragraph{$m$-Pbp}

The second planner is $m$-Pbp, which is an extension of Pbp. $m$-Pbp returns $\pi_{Pbp}|_{m} \circ \pi_{ua}^{*}$, where $\pi_{Pbp}|_{m}$ is the output path up to the $m$-th node planned by Pbp, and $\pi_{ua}^{*}$ is the unanonymized shortest path from the last node of $\pi_{Pbp}|_{m}$ to the goal while covering $t$ if $\pi_{Pbp}|_{m}$ does not cover $t$. The anonymity of $m$-Pbp relies on the anonymity of Pbp.

\begin{proposition}
\label{prop:mpbp}
    If Pbp satisfies $(k, \ell, \infty)$-Anonymity for the given domain $\mathcal{D}$, $m$-Pbp satisfies $(k, \ell, m)$-Anonymity for that domain $\mathcal{D}$.    
\end{proposition}

\paragraph{Clustering-based Planner (Cbp)}

The third $m$-bounded planner is a Clustering-based Planner (Cbp), which first applies $k$-means-like clustering to the transit candidates and then returns the concatenation of paths from $s$ to the centroid of the cluster corresponding to $t$ and from the centroid to $g$. Following~\citet{wasserman1994social}, we call a node $\sigma \in \mathcal{N}$ that minimizes the maximum distance to cover a node within a set of nodes $\mathcal{U} \subseteq \mathcal{N}$ the {\it centroid} of $\mathcal{U}$, i.e., $\sigma = \argmin_{n \in \mathcal{N}} \max_{u \in \mathcal{U}} \min_{w \in v^{-1}(u)} d(n, w)$. Like $k$-means clustering~\cite{ahmed2020k}, Cbp iteratively updates the assignment of each transit candidate to minimize the distance between each candidate and the centroid of their respective cluster. After the assignments stabilize, Cbp repeatedly merges a cluster with cardinality less than $k$ with the nearest cluster until each cluster has $k$ or more nodes. Here, we use the shortest distance from the centroid of the $i$-th cluster to the centroid of the $j$-th cluster as the distance from the $i$-th cluster to the $j$-th cluster. Then, Cbp checks $|\pi_{s \to \sigma_{t}}|$, the length of the shortest path from $s$ to the centroid of the cluster containing the true transit node $t$. 
If it exceeds $m$, Cbp assigns the first $m$ nodes from $\pi_{s \to \sigma_{t}}$, as $\pi^{1}$.
Otherwise, Cbp appends a randomly generated path as padding (Line 12-15 in Algorithm~\ref{alg:cbp}) and assigns the extended path to $\pi^{1}$.  Finally, Cbp computes $\pi^{2}$, the shortest path from the last node of $\pi^{1}$ to $g$, covering $t$, and returns $\pi^{1} \circ \pi^{2}$. The sequence of the first $m$ nodes of the output paths planned by Cbp is the same for all transit points belonging to the same cluster. Cbp satisfies the following.

\begin{proposition}
\label{prop:cbp}
Let all edges in $\mathcal{D}$ be undirected. If for any $t \in \mathscr{T}$, there exists a path from $s$ to $g$ while covering $t$, Cbp satisfies $(k, 0, m)$-Anonymity.     
\end{proposition}

\begin{algorithm}[!th]
\caption{Clustering-based Planner (Cbp)}
\label{alg:cbp}
\begin{algorithmic}[1]
\State Randomly assign the transit candidates to $[|\mathscr{T}|/k]$ clusters s.t. the number of nodes in each cluster is equal to or more than $k$ 
\While{True}
    \For{each transit candidate $t_{i} \in \mathscr{T}$}
        \State Compute the distance to each centroid
        \State Assign $t_i$ to the cluster with the nearest centroid.
    \EndFor
    \For{each cluster}
        \State Recompute the centroid for each cluster;
    \EndFor
    \If{the assignment does not change} Break
    \EndIf
\EndWhile

\While{There exists a cluster consisting of less than $k$ nodes}
\For{each cluster consisting of less than $k$ nodes}
    \State Assign all elements of this cluster to the nearest cluster
\EndFor
\EndWhile

\State $\sigma_{t} \leftarrow$ the centroid of the cluster containing $t$ 
\State $\pi^{*}_{s \to \sigma_{t}} \leftarrow$ the shortest path from $s$ to $\sigma_{t}$
\If{$|\pi^{*}_{s \to \sigma_{t}}| \geq m$} $\pi^{1} \leftarrow \pi^{*}_{s \to \sigma_{t}}|_{m}$
\Else
    \State $r = s$, $\pi_{s \to r} = s$
    \While{$|\pi_{s \to r}| + |\pi^{*}_{r \to \sigma_{t}}| < m$}
        \State $r \leftarrow$ Randomly pick the neighbour of $r$
        \State Append $r$ to the tail of $\pi_{s \to r}$
    \EndWhile
    \State $\pi^{1} \leftarrow \pi_{s \to r} \circ \pi^{*}_{r \to \sigma_{t}}$
\EndIf
\State $\pi^{2} \leftarrow$ shortest path from $\pi^{1}_{|\pi^{1}|}$ to $g$ while covering $t$  \label{cbp-wrpt}
\State 
\Return $\pi^{1} \circ \pi^{2}$
\end{algorithmic}
\end{algorithm}

\section{Experiments}
\label{sec:evaluation}

\paragraph{Benchmarks and Settings}

We evaluate the performance of Pbp on six 2D grid world benchmark instances from the  Moving AI pathfinding benchmark set~\cite{sturtevant2012benchmarks}:  den101d, den201d, lak102d, lak510d, orz000d, and orz201d. We randomly select 5 pairs of start and goal points for each benchmark. The number of transit candidates,  $|\mathscr{T}|$, is 8, 12, and 16, and we randomly select $\mathscr{T}$ from $\mathcal{N}$ for each problem. 
We use a visibility function where nodes within a range less than or equal to distance $r$ are covered (for $r$ =0,2,10). All experiments used 4-way unit cost movement. 

All algorithms were implemented in C++, and experiments were run on an Intel(R) Xeon(R) CPU E5-2650 v2 @ 2.60GHz and 125GB of RAM, running Ubuntu 22.04.2 LTS. 
A time limit of 300 seconds/instance was used.

\subsection{Evaluation of Pbp with $m = \infty$}
We evaluate Pbp using
Merge-BB and DF-BB partitioning strategies, two merge orders (Random and CostAsc), and two WRTP heuristics: the blind heuristic (equivalent to breadth-first search) and the Tunnel heuristic. We also evaluate two baseline partitioning strategies:

\paragraph{Baseline \#1: Naive}  This generates a partitioning by randomly splitting the transit candidates into pairs of 2 nodes. This approach satisfies ($2, 1, \infty$)-Anonymity for this class of undirected grid maps. 

\paragraph{Baseline \#2: Depth-First Branch-and-Bound (DF-BB) partitioning}
The DF-BB partitioning strategy starts with all nodes unassigned and then constructs a partitioning by assigning one unassigned node  to an existing subset of $\Phi$ or a new subset (details in Supp.~\ref{supp:dfbb} ~\cite{takahashi2024top-supp}). 

For the privacy parameters, we set $k$ to 2 and 3 and $\ell$ to 1 and 10.

\paragraph{Results}

\begin{table*}[!th]
\centering
\caption{Comparison of each combination of partitioning, MergeOrder, and heuristic functions on coverage, the number of evaluated partitions, and the total execution time. Merge-BB with CostAsc using Tunnel achieved the best performance.}
\begin{tabular}{@{}c|c|c|rrr|rrc|ccc|ccc@{}}
\toprule
 &
   &
  $|\mathscr{T}|$ &
  \multicolumn{1}{c}{8} &
  \multicolumn{1}{c}{12} &
  \multicolumn{1}{c|}{16} &
  \multicolumn{1}{c}{8} &
  \multicolumn{1}{c}{12} &
  16 &
  8 &
  12 &
  16 &
  8 &
  12 &
  16 \\ \midrule
  Planner  &
  \begin{tabular}[c]{@{}c@{}}MergeOrder\end{tabular} &
  Heuristic &
  \multicolumn{3}{c|}{Coverage {[}\%{]}} &
  \multicolumn{3}{c|}{\begin{tabular}[c]{@{}c@{}}Total Time {[}s{]}\end{tabular}} &
  \multicolumn{3}{c|}{APR (higher=better)} &
  \multicolumn{3}{c}{MAC (lower=better)} \\ \midrule
\multicolumn{1}{c|}{\multirow{2}{*}{Naive}} &
  \multicolumn{1}{c|}{} &
  Blind &
  \multicolumn{1}{l}{n/a} &
  \multicolumn{1}{l}{n/a} &
  \multicolumn{1}{l|}{n/a} &
  \multicolumn{1}{l}{\textless{}1} &
  \multicolumn{1}{l}{\textless{}1} &
  \multicolumn{1}{c|}{-} &
  1.00 &
  1.00 &
  1.00 &
  \multicolumn{1}{l}{0.714} &
  \multicolumn{1}{l}{0.434} &
  \multicolumn{1}{l}{0.582} \\
\multicolumn{1}{l|}{} &
  \multicolumn{1}{l|}{} &
  Tunnel &
  \multicolumn{1}{l}{n/a} &
  \multicolumn{1}{l}{n/a} &
  \multicolumn{1}{l|}{n/a} &
  \multicolumn{1}{l}{\textless{}1} &
  \multicolumn{1}{l}{\textless{}1} &
  \multicolumn{1}{c|}{-} &
  1.00 &
  1.00 &
  1.00 &
  \multicolumn{1}{l}{0.714} &
  \multicolumn{1}{l}{0.434} &
  \multicolumn{1}{l}{0.582} \\ \midrule
\multirow{2}{*}{\begin{tabular}[c]{@{}c@{}}DF-BB\end{tabular}} &
   &
  Blind &
  77 &
  0 &
  0 &
  19 &
  \multicolumn{1}{c}{-} &
  - &
  1.00 &
  0.672 &
  0.439 &
  0.232 &
  - &
  - \\
 &
   &
  Tunnel &
  100 &
  20 &
  0 &
  2 &
  92 &
  - &
  1.00 &
  0.828 &
  0.558 &
  0.229 &
  - &
  - \\ \midrule
\multirow{4}{*}{\begin{tabular}[c]{@{}c@{}}Merge-BB\end{tabular}} &
  \multirow{2}{*}{Random} &
  Blind &
  74 &
  7 &
  0 &
  35 &
  185 &
  - &
  1.00 &
  1.00 &
  1.00 &
  0.232 &
  0.418 &
  0.561 \\
 &
   &
  Tunnel &
  100 &
  27 &
  0 &
  3 &
  13 &
  - &
  1.00 &
  1.00 &
  1.00 &
  0.229 &
  0.287 &
  0.452 \\ \cmidrule(l){2-15} 
 &
  \multirow{2}{*}{CostAsc} &
  Blind &
  77 &
  7 &
  0 &
  10 &
  80 &
  - &
  1.00 &
  1.00 &
  1.00 &
  0.231 &
  0.222 &
  0.223 \\
 &
   &
  Tunnel &
  100 &
  47 &
  0 &
  1 &
  5 &
  - &
  1.00 &
  1.00 &
  1.00 &
  0.229 &
  0.190 &
  0.205 \\ \bottomrule
\end{tabular}
\label{tab:baseline-comp}
\end{table*}

Tab.~\ref{tab:baseline-comp} shows the performance of the Partitioning-based Planner when $k=2, \ell=1$. We report coverage, APR, MAC, and execution time. Coverage is the percentage of problems on which Pbp completed the search (found the optimal solution and proved its optimality) within the time limit. We report mean MAC for the configurations which found satisfying solutions (ARP=1) for all instances within the time limit. Total Time denotes the average execution time for the instances where Merge-BB with the blind heuristic completed the search for $|\mathscr{T}|$ of 8 and 12. 

Tab.~\ref{tab:relative-mac} shows the mean and standard deviation of the MAC of Merge-BB divided by the MAC of Naive, showing over 50\% improvement of Merge-BB with CostAsc.

From Tab.~\ref{tab:baseline-comp}-\ref{tab:relative-mac}, we observe that: 
(1) Pbp (Merge-BB) consistently results in better MAC than the Naive baseline, showing that searching for an optimal partition achieves significantly better path costs than a naive partitioning.
(2) Merge-BB has significantly higher coverage than DF-BB, showing that the merge-based approach is a more efficient strategy.
(3) Overall, combining Merge-BB, CostAsc, and Tunnel gives the best performance.

\begin{figure}[t]
    \centering
    \includegraphics[width=\linewidth]{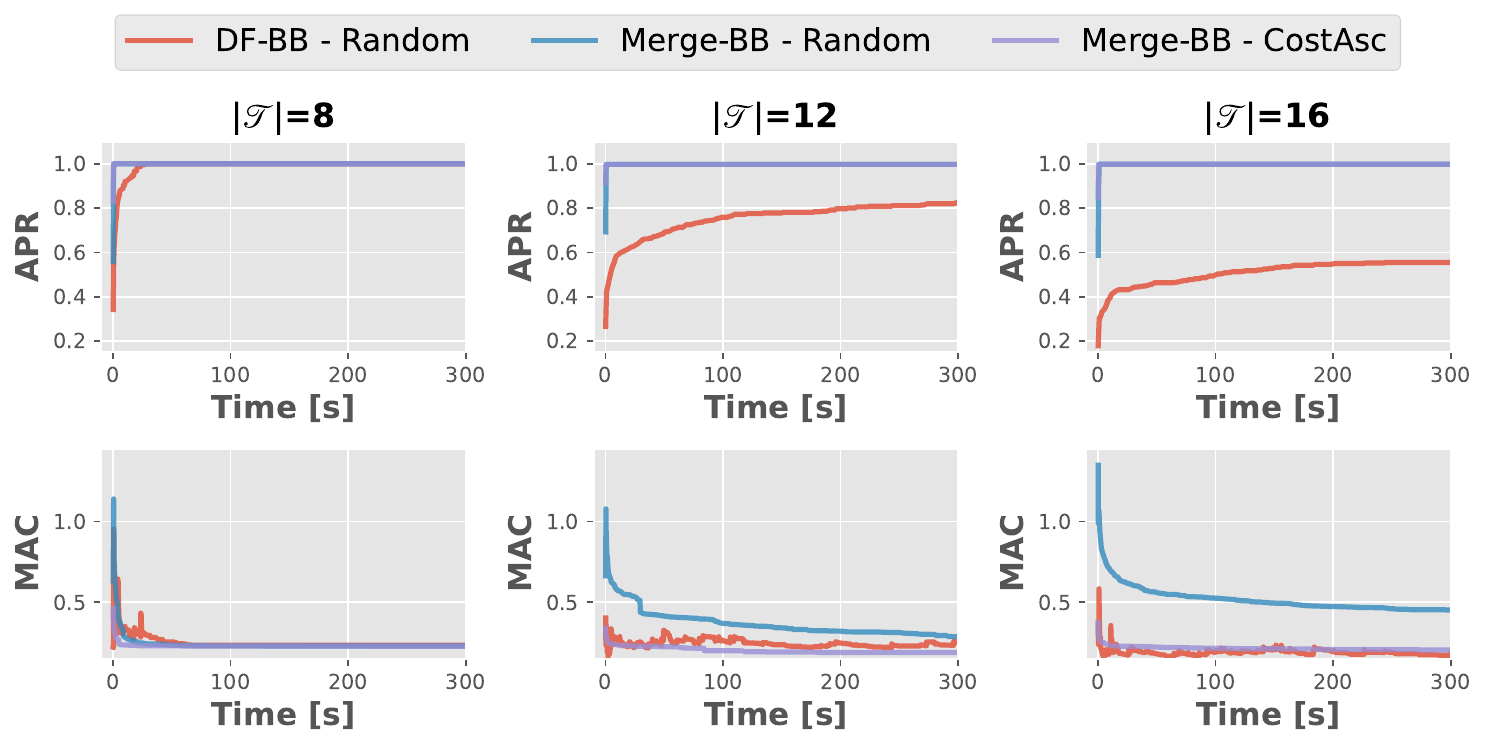}
    \caption{Convergence of ARP and MAC. Mege-BB with CostAsc shows the best performance.}
    \label{fig:mac_time}
\end{figure}

Fig.~\ref{fig:mac_time} depicts the changes over time in ARP and MAC when $k=2$, $\ell=1$. All combinations use $h_{tunnel}$. Mege-BB archives higher APR faster, and CostAsc can find solutions with lower MAC earlier.

\begin{figure}[!th]
    \centering
    \includegraphics[width=\linewidth]{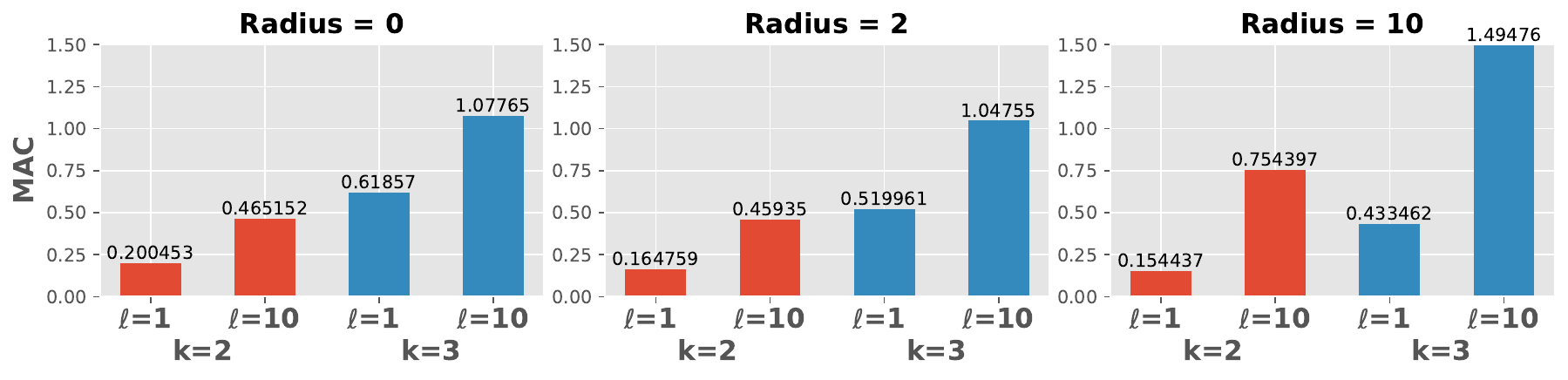}
    \caption{Impact of $k$, $\ell$, and radius $r$. Larger $k$ and $\ell$ increase MAC. The impact of $r$ is not monotonic.}
    \label{fig:impact:kl}
\end{figure}

Fig.~\ref{fig:impact:kl} shows MAC for each combination of $(k,\ell,m)$ and radius $r$ when using Merge-BB with CostAsc and $h_{tunnel}$. Larger $k$ and $\ell$ result in worse MAC. The correlation between MAC and $r$ is not monotonic since larger $r$ allows the agent to cover nodes with less movement, decreasing both the numerator and denominator of MAC.

\begin{figure}[!th]
    \centering
    \includegraphics[width=0.93\linewidth]{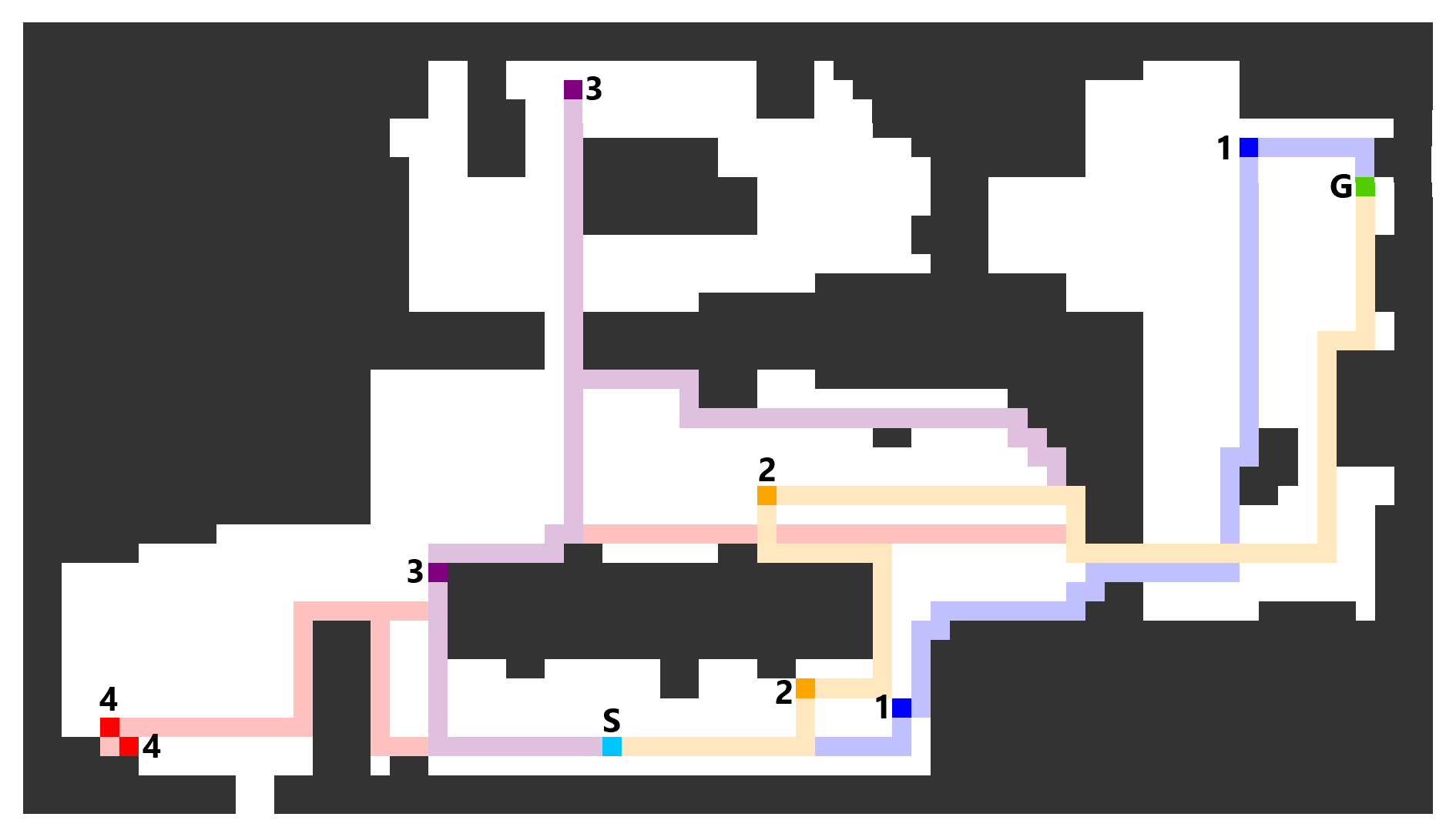}
    \caption{Example of $(2, 1, \infty)$-Anonnymized Paths. 
    }
    \label{fig:visesample}
\end{figure}

FIg.~\ref{fig:visesample} shows an example of optimal (2, 1, $\infty$)-Anonnymized Paths in \textit{den101} obtained by Pbp (Merge-BB). The cyan (S) and green (G) cells are the source and goal, respectively. The transit candidate belonging to the same subset ($1 \sim 4$) has the same color, while its corresponding path is colored in a lighter color. More qualitative examples can be found in Supp.~\ref{supp:additional-results}.

\begin{table}[!th]
\caption{Ratio of Merge-BB's MAC to Naive's MAC. Merge-BB reduces the MAC of a satisfying solution by more than 50\%.}
\begin{tabular}{@{}c|c|c|c|c@{}}
\toprule
MergeOrder &
  Heuristic &
  $|\mathscr{T}| = 8$ &
  $|\mathscr{T}| = 12$ &
  $|\mathscr{T}| = 16$ \\ \midrule
\multirow{2}{*}{Random} &
  Blind &
  \begin{tabular}[c]{@{}c@{}}0.425 \\ (±0.227)\end{tabular} &
  \begin{tabular}[c]{@{}c@{}}0.679 \\ (±0.380)\end{tabular} &
  \begin{tabular}[c]{@{}c@{}}0.937 \\ (±0.716)\end{tabular} \\ \cmidrule(l){2-5} 
 &
  Tunnel &
  \begin{tabular}[c]{@{}c@{}}0.423 \\ (±0.226)\end{tabular} &
  \begin{tabular}[c]{@{}c@{}}0.462 \\ (±0.347)\end{tabular} &
  \begin{tabular}[c]{@{}c@{}}0.706 \\ (±0.723)\end{tabular} \\ \midrule
\multirow{2}{*}{CostAsc} &
  Blind &
  \begin{tabular}[c]{@{}c@{}}0.435 \\ (±0.237)\end{tabular} &
  \begin{tabular}[c]{@{}c@{}}0.350 \\ (±0.160)\end{tabular} &
  \begin{tabular}[c]{@{}c@{}}0.378 \\ (±0.189)\end{tabular} \\ \cmidrule(l){2-5} 
 &
  Tunnel &
  \begin{tabular}[c]{@{}c@{}}0.423 \\ (±0.226)\end{tabular} &
  \begin{tabular}[c]{@{}c@{}}0.313 \\ (±0.140)\end{tabular} &
  \begin{tabular}[c]{@{}c@{}}0.349 \\ (±0.188)\end{tabular} \\ \bottomrule
\end{tabular}
\label{tab:relative-mac}
\end{table}

\subsection{Evaluation with bounded $m$}

We compared the three $m$-bounded anonymity planners, $m$-Pbp, Rbp, and Cbp, for various values of $m$ and $k$. For each problem, we set $m$ such that $m / |\pi^{*}|$ is [0.1, 0.3, 0.5, 1.0, 5.0, 10.0], where $\pi^{*}$ is the length of the shortest path for that problem. $k$ is [2, 3, 5]. The remaining parameters are kept constant: $\ell=1$, $|\mathscr{T}|=8$, and the heuristic function is $h_{tunnel}$. 

Fig.~\ref{fig:impact-m} shows that as $m$ increases, the MAC of Rbp and Cbp exhibits exponentia growth (y-axis is log scale), while $m$-Pbp's performance converges towards that of Pbp. 
$m$-Pbp aims to cover all nodes within the same partition, while Cbp strives to move towards the centroid of the cluster. Thus, when $m$ is small enough that Cbp doesn't require appending a random-walking path, the MAC of Cbp surpasses that of $m$-Pbp. However, $m$-Pbp exhibits superior performance compared to the other methods for larger values of $m$. 

Tab.~\ref{tab:impact-m-time} shows the execution time of each planner. Cbp is clearly faster than $m$-Pbp. Larger $m$ makes Cbp faster because the path after the $m$-th node tends to be shorter, reducing the runtime of the search performed by Line~\ref{cbp-wrpt} in Alg.~\ref{alg:cbp}. However, this effect decreases if $m$ is too large for Cbp to need additional time to append random nodes.

\begin{figure}[!th]
\begin{minipage}{0.5\textwidth}
\centering
\captionof{table}{Comparison of the runtime ([s]) for each planner. Rbp and Cbp show better scalability compared to $m$-Pbp.
}
\label{tab:impact-m-time}
\begin{tabular}{@{}c|c|c|c|c|c|c@{}}
\toprule
$m / |\pi^{*}|$ &  0.1    & 0.3    & 0.5    & 1.0    & 5.0    & 10.0    \\ \midrule
Rbp   & 0.006  & 0.006  & 0.006  & 0.006  & 0.006  & 0.006  \\
Cbp   & 0.098  & 0.085  & 0.046  & 0.009  & 0.006  & 0.006  \\
$m$-Pbp & 63.985 & 63.985 & 63.985 & 63.985 & 63.985 & 63.985 \\ \bottomrule
\end{tabular}
\vspace{3.0mm}
\end{minipage}
\hfill
\begin{minipage}{0.5\textwidth}
\centering 
\includegraphics[width=\linewidth]{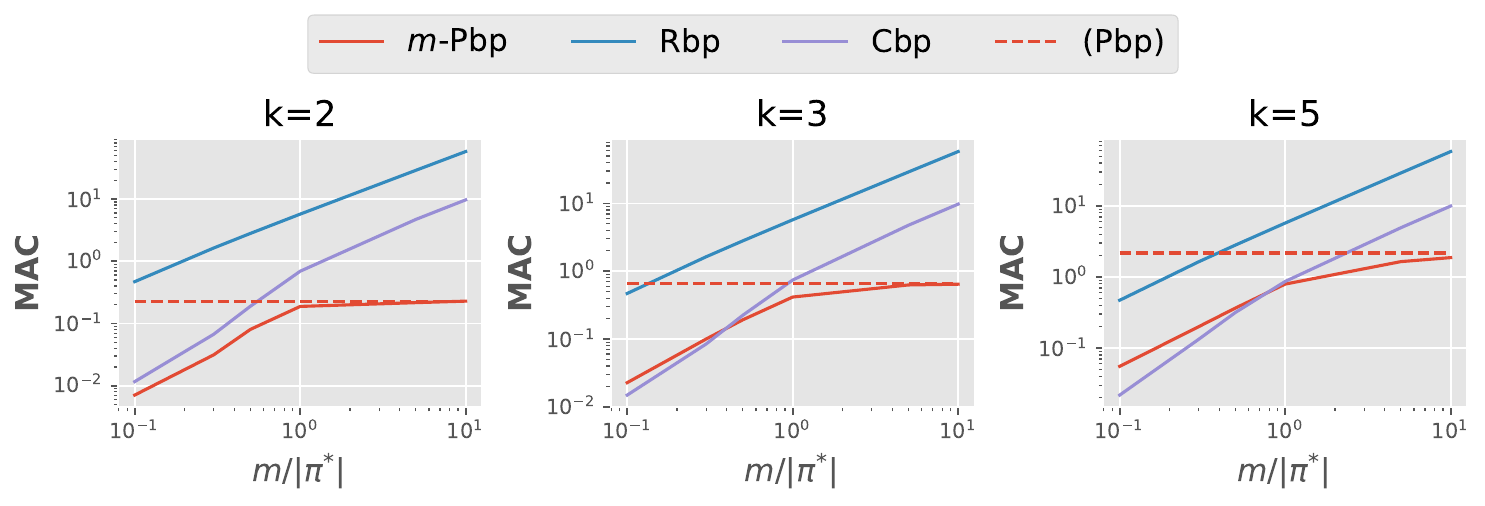}
\caption{Impact of $k$ and $m$ on MAC. While MAC of Rbp increases linearly with respect to $m$, MAC of $m$-Pbp converges to MAC of Pbp.}
\label{fig:impact-m}
\end{minipage}
\end{figure}

\section{Related Work}
\label{sec:related-work}

\paragraph{$k$-Anonymity}

$k$-anonymity is a fundamental concept in data privacy and anonymization that aims to safeguard individual identities in a dataset while preserving its utility~\cite{sweeney2002k,terzi2015survey}.  $k$-anonymity seeks to render each record in a dataset indistinguishable from at least $k$-1 other records, i.e., each individual's data is grouped with a minimum of $k$-1 other individuals with similar attributes. This grouping makes it challenging to identify a specific individual within the group.

\paragraph{Obfuscation}

Some existing methods for goal obfuscation leverage concepts similar to $k$-anonymity. For instance,~\citet{kulkarni2018resource,kulkarni2019unified} define a secure path as one where $k$ different goals result in the same path, thereby making it difficult for an observer to discern the true purpose among these $k$ nodes. Another example is \textit{Dissimulation} proposed in~\citet{masters2017deceptive}, where the true goal is considered to be obfuscated if there exist other nodes that look like the goal equally or more than the real goal. While our work refrains from making any assumptions regarding how the observer deduces the agent's intention, some studies ~\citep{masters2017deceptive,inproceedings,lewis2023deceptive,savas2022deceptiveuncertainty} model the inference process of the observer and devise obfuscation techniques tailored to these models. However, approaches based on such models do not provide a guarantee of security against adversaries who do not adhere to the model assumptions.

\section{Conclusion}
\label{sec:conclusion}

This paper introduced the Transit Obfuscation Problem and proposed novel techniques to address this challenge. We introduced $(k, \ell, m)$-Anonymity as a measure of concealment achieved by a path planner.
We proposed a Pbp, a partitioning based algorithm to achieve $(k, \ell, \infty)$-Anonymity, and evaluated its performance on 2D grid maps with obstacles.
We showed that Pbp with a merge-based branch-and-bound strategy significantly outperforms baseline partitioning approaches.

Although we showed that Merge-BB with the APR and MAC objectives is a viable approach to partitioning, a complete branch-and-bound search to find and prove the optimality of a solution poses a scalability challenge. For example, although our current implementation of Pbp can find solutions for $|\mathscr{T}| =16$ (Tab.~\ref{tab:baseline-comp}), it can not complete the search and prove optimality within the 300 sec. limit. Search algorithms finding good partitionings quickly without an optimality guarantee (e.g., local search/metaheuristics) are a direction for future work.

We also investigated algorithms for $m$-bounded anonymity ($m < \infty$). We showed that while $m$-Pbp yielded the best MAC scores, Cbp offers fairly good MAC scores but runs much faster. Future work will investigate additional approaches to trading off scalability vs. solution quality.

Finally, while this work focused on a single agent and single adversary in a static environment, an extension of our proposed techniques to more complex scenarios, such as multi-agent systems and dynamic environments, is another direction for future work.



\bibliographystyle{ACM-Reference-Format} 
\bibliography{ref}


\begin{thebibliography}{20}


\ifx \showCODEN    \undefined \def \showCODEN     #1{\unskip}     \fi
\ifx \showDOI      \undefined \def \showDOI       #1{#1}\fi
\ifx \showISBNx    \undefined \def \showISBNx     #1{\unskip}     \fi
\ifx \showISBNxiii \undefined \def \showISBNxiii  #1{\unskip}     \fi
\ifx \showISSN     \undefined \def \showISSN      #1{\unskip}     \fi
\ifx \showLCCN     \undefined \def \showLCCN      #1{\unskip}     \fi
\ifx \shownote     \undefined \def \shownote      #1{#1}          \fi
\ifx \showarticletitle \undefined \def \showarticletitle #1{#1}   \fi
\ifx \showURL      \undefined \def \showURL       {\relax}        \fi
\providecommand\bibfield[2]{#2}
\providecommand\bibinfo[2]{#2}
\providecommand\natexlab[1]{#1}
\providecommand\showeprint[2][]{arXiv:#2}

\bibitem[\protect\citeauthoryear{Ahmed, Seraj, and Islam}{Ahmed et~al\mbox{.}}{2020}]%
        {ahmed2020k}
\bibfield{author}{\bibinfo{person}{Mohiuddin Ahmed}, \bibinfo{person}{Raihan Seraj}, {and} \bibinfo{person}{Syed Mohammed~Shamsul Islam}.} \bibinfo{year}{2020}\natexlab{}.
\newblock \showarticletitle{The k-means algorithm: A comprehensive survey and performance evaluation}.
\newblock \bibinfo{journal}{\emph{Electronics}} \bibinfo{volume}{9}, \bibinfo{number}{8} (\bibinfo{year}{2020}), \bibinfo{pages}{1295}.
\newblock


\bibitem[\protect\citeauthoryear{Chakraborti, Kulkarni, Sreedharan, Smith, and Kambhampati}{Chakraborti et~al\mbox{.}}{2019}]%
        {chakraborti2019explicability}
\bibfield{author}{\bibinfo{person}{Tathagata Chakraborti}, \bibinfo{person}{Anagha Kulkarni}, \bibinfo{person}{Sarath Sreedharan}, \bibinfo{person}{David~E Smith}, {and} \bibinfo{person}{Subbarao Kambhampati}.} \bibinfo{year}{2019}\natexlab{}.
\newblock \showarticletitle{Explicability? legibility? predictability? transparency? privacy? security? the emerging landscape of interpretable agent behavior}. In \bibinfo{booktitle}{\emph{Proceedings of the international conference on automated planning and scheduling}}, Vol.~\bibinfo{volume}{29}. \bibinfo{pages}{86--96}.
\newblock


\bibitem[\protect\citeauthoryear{Chen, Fung, Desai, and Sossou}{Chen et~al\mbox{.}}{2012}]%
        {chen2012differentially}
\bibfield{author}{\bibinfo{person}{Rui Chen}, \bibinfo{person}{Benjamin~CM Fung}, \bibinfo{person}{Bipin~C Desai}, {and} \bibinfo{person}{N{\'e}riah~M Sossou}.} \bibinfo{year}{2012}\natexlab{}.
\newblock \showarticletitle{Differentially private transit data publication: a case study on the montreal transportation system}. In \bibinfo{booktitle}{\emph{Proceedings of the 18th ACM SIGKDD international conference on Knowledge discovery and data mining}}. \bibinfo{pages}{213--221}.
\newblock


\bibitem[\protect\citeauthoryear{Dias, Aylett, Paiva, and Reis}{Dias et~al\mbox{.}}{2013}]%
        {dias2013great}
\bibfield{author}{\bibinfo{person}{Joao Dias}, \bibinfo{person}{Ruth Aylett}, \bibinfo{person}{Ana Paiva}, {and} \bibinfo{person}{Henrique Reis}.} \bibinfo{year}{2013}\natexlab{}.
\newblock \showarticletitle{The great deceivers: Virtual agents and believable lies}. In \bibinfo{booktitle}{\emph{Proceedings of the Annual Meeting of the Cognitive Science Society}}, Vol.~\bibinfo{volume}{35}.
\newblock


\bibitem[\protect\citeauthoryear{Enayati, Goeckel, Houmansadr, and Pishro-Nik}{Enayati et~al\mbox{.}}{2022}]%
        {9851770}
\bibfield{author}{\bibinfo{person}{Saeede Enayati}, \bibinfo{person}{Dennis~L. Goeckel}, \bibinfo{person}{Amir Houmansadr}, {and} \bibinfo{person}{Hossein Pishro-Nik}.} \bibinfo{year}{2022}\natexlab{}.
\newblock \showarticletitle{Privacy-Preserving Path-Planning for UAVs}. In \bibinfo{booktitle}{\emph{2022 International Symposium on Networks, Computers and Communications (ISNCC)}}. \bibinfo{pages}{1--6}.
\newblock
\urldef\tempurl%
\url{https://doi.org/10.1109/ISNCC55209.2022.9851770}
\showDOI{\tempurl}


\bibitem[\protect\citeauthoryear{Keren, Gal, and Karpas}{Keren et~al\mbox{.}}{2016}]%
        {keren2016privacy}
\bibfield{author}{\bibinfo{person}{Sarah Keren}, \bibinfo{person}{Avigdor Gal}, {and} \bibinfo{person}{Erez Karpas}.} \bibinfo{year}{2016}\natexlab{}.
\newblock \showarticletitle{Privacy Preserving Plans in Partially Observable Environments.}. In \bibinfo{booktitle}{\emph{IJCAI}}. \bibinfo{pages}{3170--3176}.
\newblock


\bibitem[\protect\citeauthoryear{Kulkarni, Klenk, Rane, and Soroush}{Kulkarni et~al\mbox{.}}{2018}]%
        {kulkarni2018resource}
\bibfield{author}{\bibinfo{person}{Anagha Kulkarni}, \bibinfo{person}{Matthew Klenk}, \bibinfo{person}{Shantanu Rane}, {and} \bibinfo{person}{Hamed Soroush}.} \bibinfo{year}{2018}\natexlab{}.
\newblock \showarticletitle{Resource bounded secure goal obfuscation}. In \bibinfo{booktitle}{\emph{AAAI Fall Symposium on Integrating Planning, Diagnosis and Causal Reasoning}}.
\newblock


\bibitem[\protect\citeauthoryear{Kulkarni, Srivastava, and Kambhampati}{Kulkarni et~al\mbox{.}}{2019}]%
        {kulkarni2019unified}
\bibfield{author}{\bibinfo{person}{Anagha Kulkarni}, \bibinfo{person}{Siddharth Srivastava}, {and} \bibinfo{person}{Subbarao Kambhampati}.} \bibinfo{year}{2019}\natexlab{}.
\newblock \showarticletitle{A unified framework for planning in adversarial and cooperative environments}. In \bibinfo{booktitle}{\emph{Proceedings of the AAAI Conference on Artificial Intelligence}}, Vol.~\bibinfo{volume}{33}. \bibinfo{pages}{2479--2487}.
\newblock


\bibitem[\protect\citeauthoryear{Lewis and Miller}{Lewis and Miller}{2023}]%
        {lewis2023deceptive}
\bibfield{author}{\bibinfo{person}{Alan Lewis} {and} \bibinfo{person}{Tim Miller}.} \bibinfo{year}{2023}\natexlab{}.
\newblock \showarticletitle{Deceptive Reinforcement Learning in Model-Free Domains}.
\newblock \bibinfo{journal}{\emph{arXiv preprint arXiv:2303.10838}} (\bibinfo{year}{2023}).
\newblock


\bibitem[\protect\citeauthoryear{Luo, Zhang, Xiang, and Jiongming}{Luo et~al\mbox{.}}{2019}]%
        {inproceedings}
\bibfield{author}{\bibinfo{person}{Junren Luo}, \bibinfo{person}{Wanpeng Zhang}, \bibinfo{person}{Fengtao Xiang}, {and} \bibinfo{person}{Su Jiongming}.} \bibinfo{year}{2019}\natexlab{}.
\newblock \showarticletitle{Intention Obfuscated Adversarial Deceptive Path Recommendation for UGV Patrol Maneuver}. \bibinfo{pages}{206--211}.
\newblock
\urldef\tempurl%
\url{https://doi.org/10.1109/IHMSC.2019.00055}
\showDOI{\tempurl}


\bibitem[\protect\citeauthoryear{Masters and Sardina}{Masters and Sardina}{2017}]%
        {masters2017deceptive}
\bibfield{author}{\bibinfo{person}{Peta Masters} {and} \bibinfo{person}{Sebastian Sardina}.} \bibinfo{year}{2017}\natexlab{}.
\newblock \showarticletitle{Deceptive Path-Planning.}. In \bibinfo{booktitle}{\emph{IJCAI}}. \bibinfo{pages}{4368--4375}.
\newblock


\bibitem[\protect\citeauthoryear{Price, Fraga~Pereira, Masters, and Vered}{Price et~al\mbox{.}}{2023}]%
        {price2023domain}
\bibfield{author}{\bibinfo{person}{Adrian Price}, \bibinfo{person}{Ramon Fraga~Pereira}, \bibinfo{person}{Peta Masters}, {and} \bibinfo{person}{Mor Vered}.} \bibinfo{year}{2023}\natexlab{}.
\newblock \showarticletitle{Domain-Independent Deceptive Planning}. In \bibinfo{booktitle}{\emph{Proceedings of the 2023 International Conference on Autonomous Agents and Multiagent Systems}}. \bibinfo{pages}{95--103}.
\newblock


\bibitem[\protect\citeauthoryear{Savas, Verginis, and Topcu}{Savas et~al\mbox{.}}{2022}]%
        {savas2022deceptiveuncertainty}
\bibfield{author}{\bibinfo{person}{Yagiz Savas}, \bibinfo{person}{Christos~K Verginis}, {and} \bibinfo{person}{Ufuk Topcu}.} \bibinfo{year}{2022}\natexlab{}.
\newblock \showarticletitle{Deceptive decision-making under uncertainty}. In \bibinfo{booktitle}{\emph{Proceedings of the AAAI Conference on Artificial Intelligence}}, Vol.~\bibinfo{volume}{36}. \bibinfo{pages}{5332--5340}.
\newblock


\bibitem[\protect\citeauthoryear{Seiref, Jaffey, Lopatin, and Felner}{Seiref et~al\mbox{.}}{2020}]%
        {seiref2020solving}
\bibfield{author}{\bibinfo{person}{Shawn Seiref}, \bibinfo{person}{Tamir Jaffey}, \bibinfo{person}{Margarita Lopatin}, {and} \bibinfo{person}{Ariel Felner}.} \bibinfo{year}{2020}\natexlab{}.
\newblock \showarticletitle{Solving the watchman route problem on a grid with heuristic search}. In \bibinfo{booktitle}{\emph{Proceedings of the international conference on automated planning and scheduling}}, Vol.~\bibinfo{volume}{30}. \bibinfo{pages}{249--257}.
\newblock


\bibitem[\protect\citeauthoryear{Skyler, Atzmon, Yaffe, and Felner}{Skyler et~al\mbox{.}}{2022}]%
        {SkylerAYF22}
\bibfield{author}{\bibinfo{person}{Shawn Skyler}, \bibinfo{person}{Dor Atzmon}, \bibinfo{person}{Tamir Yaffe}, {and} \bibinfo{person}{Ariel Felner}.} \bibinfo{year}{2022}\natexlab{}.
\newblock \showarticletitle{Solving the Watchman Route Problem with Heuristic Search}.
\newblock \bibinfo{journal}{\emph{J. Artif. Intell. Res.}}  \bibinfo{volume}{75} (\bibinfo{year}{2022}), \bibinfo{pages}{747--793}.
\newblock
\urldef\tempurl%
\url{https://doi.org/10.1613/jair.1.13685}
\showDOI{\tempurl}


\bibitem[\protect\citeauthoryear{Sturtevant}{Sturtevant}{2012}]%
        {sturtevant2012benchmarks}
\bibfield{author}{\bibinfo{person}{N. Sturtevant}.} \bibinfo{year}{2012}\natexlab{}.
\newblock \showarticletitle{Benchmarks for Grid-Based Pathfinding}.
\newblock \bibinfo{journal}{\emph{Transactions on Computational Intelligence and AI in Games}} \bibinfo{volume}{4}, \bibinfo{number}{2} (\bibinfo{year}{2012}), \bibinfo{pages}{144 -- 148}.
\newblock
\urldef\tempurl%
\url{http://web.cs.du.edu/~sturtevant/papers/benchmarks.pdf}
\showURL{%
\tempurl}


\bibitem[\protect\citeauthoryear{Sweeney}{Sweeney}{2002}]%
        {sweeney2002k}
\bibfield{author}{\bibinfo{person}{Latanya Sweeney}.} \bibinfo{year}{2002}\natexlab{}.
\newblock \showarticletitle{k-anonymity: A model for protecting privacy}.
\newblock \bibinfo{journal}{\emph{International journal of uncertainty, fuzziness and knowledge-based systems}} \bibinfo{volume}{10}, \bibinfo{number}{05} (\bibinfo{year}{2002}), \bibinfo{pages}{557--570}.
\newblock


\bibitem[\protect\citeauthoryear{Takahashi and Fukunaga}{Takahashi and Fukunaga}{2024}]%
        {takahashi2024top-supp}
\bibfield{author}{\bibinfo{person}{Hideaki Takahashi} {and} \bibinfo{person}{Alex Fukunaga}.} \bibinfo{year}{2024}\natexlab{}.
\newblock \showarticletitle{Supplementary Material for "On the Transit Obfuscation Problem}.
\newblock \bibinfo{journal}{\emph{arXiv preprint}} (\bibinfo{year}{2024}).
\newblock


\bibitem[\protect\citeauthoryear{Terzi, Terzi, and Sagiroglu}{Terzi et~al\mbox{.}}{2015}]%
        {terzi2015survey}
\bibfield{author}{\bibinfo{person}{Duygu~Sinanc Terzi}, \bibinfo{person}{Ramazan Terzi}, {and} \bibinfo{person}{Seref Sagiroglu}.} \bibinfo{year}{2015}\natexlab{}.
\newblock \showarticletitle{A survey on security and privacy issues in big data}. In \bibinfo{booktitle}{\emph{2015 10th International Conference for Internet Technology and Secured Transactions (ICITST)}}. IEEE, \bibinfo{pages}{202--207}.
\newblock


\bibitem[\protect\citeauthoryear{Wasserman and Faust}{Wasserman and Faust}{1994}]%
        {wasserman1994social}
\bibfield{author}{\bibinfo{person}{Stanley Wasserman} {and} \bibinfo{person}{Katherine Faust}.} \bibinfo{year}{1994}\natexlab{}.
\newblock \showarticletitle{Social network analysis: Methods and applications}.
\newblock  (\bibinfo{year}{1994}).
\newblock


\end{thebibliography}

\makeatletter\@input{suppaux.tex}\makeatother


\end{document}





\pagestyle{fancy}
\fancyhead{}

\maketitle 

\appendix

\beginsupp

\section{Supplement for Sec.~\ref{sec:TA}}
\label{supp:proof:sec3}

\subsection{Proofs for Sec.~\ref{sec:TA}}

\begin{proof}[Proof of Prop.~\ref{thm:pathcoverage-path} (3C Condition for Output Path)]

$\hquad$ We will provide the necessary and sufficient conditions.

\textit{Necessary Condition:} Suppose $\pi= \mathcal{A}(\langle \mathcal{D}, s, g, t \rangle)$ is a $(k, \ell, \infty)$-Anonymized Path from $s$ to $g$. This means that there exists a set $T \subseteq \{t' | t' \in \mathscr{T} \text{ s.t. } \pi = \mathcal{A}(\langle \mathcal{D}, s, g, t' \rangle) \}$ satisfying $|T| \geq k$ and $\min_{(i, j ) \in T \times T } d(i, j) \geq \ell$, which ensures conditions (1) and (2) are met. In addition, the first node of $\pi$ is $s$, and the last node f $\pi$ is $g.$. Then, if condition (3) does not hold, it implies that $\pi$ does not cover all nodes in $T$, and there exists a node $t' \in T$ that is not covered by $\pi$. However, since $\mathcal{A}(\langle \mathcal{D}, s, g, t' \rangle)$ covers $t'$, this contradicts the assumption that $\pi = \mathcal{A}(\langle \mathcal{D}, s, g, t' \rangle)$.

\textit{Sufficient Condition:} Suppose there exists a set of nodes $T \subseteq \mathscr{T}$ that satisfies conditions (1), (2), and (3) for $\pi$. Then, there exists a set of nodes $T \subseteq \{t' | t' \in \mathscr{T} \text{ s.t. } \pi = \mathcal{A}(\langle \mathcal{D}, s, g, t \rangle) = \mathcal{A}(\langle \mathcal{D}, s, g, t' \rangle) \}$ with $|T| \geq k$ and $\min_{(i, j ) \in T \times T } d(i, j) \geq \ell$. Hence, $\pi$ is a $(k, \ell, \infty)$-Anonymizable Path.

\end{proof}

\begin{proof}[Proof of Lemma~\ref{lem:path-extension} (Path-Extension)]
By Definition~\ref{def:transit-anonymized-path}, there exists a set $T_{\mathcal{A}} \subseteq \{t' $$ | t' $$ \in $$ \mathscr{T} \text{ s.t. } \mathcal{A}(\langle $$ \mathcal{D}, $$ s, $$ g, $$ t \rangle)|_m $ $=$ $ \mathcal{A}(\langle $ $ \mathcal{D}, s, $$ g, $$ t' \rangle)|_m $$ \}$ such that $|T_\mathcal{A}| \geq k \hquad \text{and} \hquad \min_{(i, j ) \in T \times T} d(i, j) \geq \ell$. Let $B$ be a path planner that takes a tuple $\langle \mathcal{D}, S', g', t \rangle$ and returns the concatenation of paths from $s'$ to $s$, $\mathcal{A}(\langle \mathcal{D}, s, g, t \rangle)$, and $g$ to $g'$. Then, there exists a set $T_{\mathcal{B}} \subseteq \{t' | t' \in \mathscr{T} \text{ s.t. } \mathcal{B}(\langle \mathcal{D}, s', g', t \rangle) = \mathcal{B}(\langle \mathcal{D}, s', g', t' \rangle) \}$ such that $|T_{\mathcal{B}}| \geq k$ and $\hquad \min_{(i, j ) \in T \times T} d(i, j) \geq \ell$.
\end{proof}

\subsection{Additional Properties of $(k, \ell, m)$-Anonymity}
\label{supp:additional-properties}

First, $(k, \ell, m)$-Anonymizable Tuple and Anonymized Path satisfy \textit{monotonicity}.

\begin{lemma}[Monotonicity of Input and Output]
\label{lemma:mono}
Let $(k, \ell, m)$ and $(k', \ell', m')$ be an arbitrary pair of parameters such that $k' \leq  k$, $\ell' \leq \ell$, and $m' \leq m$. Then, for any domain, $(k, \ell, m)$-Anonymizable Tuple is $(k', \ell', m')$-Anonymizable Tuple, and $(k, \ell, m)$-Anonymized Path is $(k', \ell', m')$-Anonymized Path
\end{lemma}

\begin{proof}[Proof of Lemma~\ref{lemma:mono} ((Monotonicity of Input and Output)]

If the input is a $(k, \ell, m)$-Anonymizable Tuple, there exists a path planner $\mathcal{A}$ such that there exists a set $T \subseteq \{ t' | t' \in \mathscr{T} \text{ and } \mathcal{A}(\langle \mathcal{D}, s, g, t \rangle)|_{m'}$ $=\mathcal{A}(\langle \mathcal{D}, s, g, t' \rangle)|_{m'} \}$ satisfying $|T| \geq k \geq k'$ and $\min_{(i, j) \in T \times T}$ $d(i, j) \geq \ell \geq \ell'$. Likewise, if the output path is a $(k, \ell, m)$-Anonymized Path planned by $\mathcal{A}$, there exists a set $T \subseteq \{ t' | t' \in \mathscr{T} \text{ and } \mathcal{A}(\langle \mathcal{D}, s, g,$ $t \rangle)|_{m'} = \mathcal{A}(\langle \mathcal{D}, s, g, t' \rangle)|_{m'} \}$ satisfying $|T| \geq k \geq k'$ and $\min_{(i, j) \in T \times T} $ $ d(i, j) \geq \ell \geq \ell'$. 
\end{proof}

The necessary and sufficient conditions for an input tuple to be $(k, \ell, \infty)$-Anonymizable Tuple are as follows.

\begin{proposition}[3C Condition for Input Tuple]
\label{thm:pathcoverage}
A tuple $\langle \mathcal{D}, s, g, t \rangle$ is $(k, \ell, \infty)$-Anonymizable iff there exists a set of nodes $T \subseteq \mathscr{T}$ that satisfies the following conditions:
\begin{enumerate}
\item Cardinality: $|T| \geq k$
\item Cost: $\min_{(i, j ) \in T \times T } d(i, j) \geq \ell$
\item Coverage: There exists a path $\pi$ from $s$ to $g$ that covers all nodes in $T$
\end{enumerate}
\end{proposition}

\begin{proof}[Proof of Prop~\ref{thm:pathcoverage} (3C Condition for Input Tuple)]

$\hquad$ We will provide the necessary and sufficient conditions.

\textit{Necessary Condition:} Let $\pi$ $= \mathcal{A}(\langle \mathcal{D}, s, g, t \rangle)$. If $\langle \mathcal{D}, s, g, t \rangle$ is a $(k, \ell, \infty)$-Anonymizable Tuple, there exists a path planner $\mathcal{A}$ such that there exists a set $T \subseteq \{t' | t' \in \mathscr{T} \text{ s.t. } \pi = \mathcal{A}(\langle \mathcal{D}, s, g, t' \rangle) \}$ satisfying $|T| \geq k$ and $\min_{(i, j ) \in T \times T } d(i, j) \geq \ell$, which ensures conditions (1) and (2) are met. If condition (3) does not hold, it implies that $\pi$ does not cover all nodes in $T$, and there exists a node $t' \in T$ that is not covered by $\pi$. However, since $\mathcal{A}(\langle \mathcal{D}, s, g, t' \rangle)$ covers $t'$, this contradicts the assumption that $\pi = \mathcal{A}(\langle \mathcal{D}, s, g, t' \rangle)$.

\textit{Sufficient Condition:} Suppose there exists a set of nodes $T \subseteq \mathscr{T}$ that satisfies conditions (1), (2), and (3) for $\langle \mathcal{D}, s, g, t \rangle$. Then, there exists a path $\pi$ from $s$ to $g$ that covers all nodes within $T$. If a path planner $\mathcal{A}$ returns $\pi$ when taking any node within $T$ as the transit node, this planner satisfies $T \subseteq \{t' | t' \in \mathscr{T} \text{ s.t. } \mathcal{A}(\langle \mathcal{D}, s, g, t \rangle) = \mathcal{A}(\langle \mathcal{D}, s, g, t' \rangle) \}$ with $|T| \geq k$ and $\min_{(i, j ) \in T \times T } d(i, j) \geq \ell$. Hence, $\langle D, s, g, t \rangle$ is a $(k, \ell, \infty)$-Anonymizable Tuple.

\end{proof}

If the domain consists of an undirected graph, we can also ensure the monotonicity of $(k, \ell, m)$-Anonymity.

\begin{lemma}[Monotonicity of Planner]
\label{lemma:mono-planner}
If all edges in the domain are undirected and $m' \leq m$, a planner satisfying $(k, \ell, m)$-Anonymity achieves $(k, \ell, m')$-Anonymity.
\end{lemma}

\begin{proof}[Proof of Lemma.~\ref{lemma:mono-planner} (Monotonicity of Planner)]
We'll start by proving that an $(k, \ell, m')$-Anonymizable Tuple also qualifies as an $(k, \ell, m)$-Anonymizable Tuple, where $m$ and $m'$ are integers or infinity, and $m \geq m'$.

For an input tuple to be considered $(k, \ell, m)$-Anonymizable (with $m \geq 1$), it must satisfy certain conditions: there should exist a set $T \subseteq \mathscr{T}$ with at least $k$ elements, the minimum distance between any pair $(i, j) \in T \times T$ should be $\ell$ or more, and for each element $t$ in $T$, a valid path from $s$ to $g$ covering $t$ must exist.

In cases where all edges in the domain are undirected, a path from $s$ to $g$ covering all nodes in $T$ is feasible. Thus, by invoking Prop.~\ref{thm:pathcoverage}, we conclude that any $(k, \ell, m')$-Anonymizable Tuple is $(k, \ell, \infty)$-Anonymizable.

Then, a path planner satisfying $(k, \ell, m)$-Anonymity, which outputs $(k, \ell, m)$-Anonymized Path for any $(k, \ell, m)$-Anonymizable Tuple, also outputs $(k, \ell, m')$-Anonymized Path for any $(k, \ell, m')$-Anonymizable Tuple. Here, we use the monotonicity of $(k, \ell, m)$-Anonymized Path and $(k, \ell, m)$-Anonymizable Tuple.
\end{proof}

\section{Proofs for Sec.~\ref{sec:partitioning}}
\label{supp:proof:sec4}

\begin{lemma}
\label{lem:local-kl-ta}
Let $|\mathscr{T}_{(s, g)}|$ be the number of transit points whose input tuples with $s$ and $g$ are $(k, \ell, \infty)$-Anonymizable Tuples. Then, if $|\mathscr{T}_{(s, g)}| \geq 1$, Alg.~\ref{alg:1} satisfies $(k, \ell, \infty, \sum^{\Phi-1}_{\phi=1} |T_{\phi}| / |\mathscr{T}_{(s, g)}|)$-Local Anonymity for any combination of $s$, $g$, and $\mathcal{D}$.
\end{lemma}

\begin{proof}[Proof of Lemma~\ref{lem:local-kl-ta}]
Given the constitution of $T_{\phi}$, $\phi = 1, 2, ..., \Phi-1$ and Theorem~2, any output path is $(k, \ell)$-Anonymized Path if the $t$ belongs to one of $T_{\phi}$, $\phi = 1, 2, ..., \Phi-1$.
\end{proof}

\begin{proof}[Proof of Theorem~\ref{thm:comp} (Completeness)] Direct application of $\\$ Lemma~\ref{lem:local-kl-ta} and Theorem~\ref{thm:existence}.
\end{proof}

\begin{proof}[Proof of Prop.~\ref{prop:opt} (Optimality)]
Suppose there exists a path planner $\mathcal{A}'$ that gives a larger APR. In this case, we can partition nodes by gathering nodes for which $\mathcal{A}'$ outputs the same path to a subset. However, Alg.~\ref{alg:1} using Alg.~\ref{alg:mergebb} can find this partition and should select it. Additionally, Alg.~\ref{alg:1} with Alg.~\ref{alg:mergebb} always finds the shortest path for each transit point and selects the partition with the minimum average cost, resulting in the smallest MAC.
\end{proof} 

\section{Proofs for Sec.~\ref{sec:m-bounded}}
\label{supp:proof:sec5}

\begin{proof}[Proof of Prop.~\ref{prop:rbp} (($k, \ell, m$)-Anonymity of Rbp)]
Let $\mathcal{A}_{Rbp}$ be Rbp. We first show that if all edges are undirected, Rbp does not return False for any $(k, \ell, m)$-Anonymizable Tuple. Since there exists a path from the source $s$ and the goal $g$ while covering the transit node $t$, and $s$ and $g$ are different nodes, Rbp can continue randomly picking at least one neighbor when planning the first $m$ nodes. It is also obvious that there exists a path from the source $s$ and the goal $g$ while covering the transit node $t$ if the domain consists of an undirected graph. Recall that the first $m$ nodes of the output path planned by Rbp are the same except in the case of the output being Failure. Then, for any $(k, \ell, m)$-Anonymizable Tuple $\langle D, s, g, t \rangle$, there exists a set $T \subseteq \mathscr{T}$ such that $t \in T$, $|T| \geq k$ and $\min_{(i, j) \in T \times T } d(i, j) \geq \ell$. Then, we have that for any $(k, \ell, m)$-Anonymizable Tuple $\langle D, s, g, t \rangle$, there exists a set $T \subseteq \{ t' | t' \in \mathscr{T} \text{ and } \mathcal{A}_{Rbp}(\langle \mathcal{D}, s, g, t \rangle)|_{m} = \mathcal{A}_{Rbp}(\langle \mathcal{D}, s, g, t' \rangle)|_{m} \}$ satisfying $|T| \geq k$ and $\min_{(i, j) \in T \times T } d(i, j) \geq \ell$, which means that the output path for any $(k, \ell, m)$-Anonymizable Tuple is always $(k, \ell, m)$-Anonymized Path.
\end{proof}

\begin{proof}[Proof of Prop.~\ref{prop:mpbp} (($k, \ell, m$)-Anonymity of $m$-Pbp)]
Direct application of Lemma.~\ref{lemma:mono-planner}.
\end{proof}

\begin{proof}[Proof of Prop.~\ref{prop:cbp} (($k, \ell, m$)-Anonymity of Cbp)]
If all edges in $\mathcal{D}$ are undirected, and there exists a path from $s$ to $g$ while covering $t$ for any $t \in \mathscr{T}$, there always exists a path from $s$ to $g$ via $\sigma_{t}$, and $t$ is coverable from the middle of that path. In addition, the first $m$ nodes planned by Cbp are the same for at least $k$ nodes, whose minimum distance is $0$.
\end{proof}

\section{DF-BB pseudocode}
\label{supp:dfbb}

Alg.~\ref{alg:dfbb} shows the pseudo-code of Depth-BB partitioning. The set of unassigned nodes is represented as $\mathscr{T}^{U}$.

\begin{algorithm}[!th]
\caption{Depth-BB Partioning}
\label{alg:dfbb}
\begin{algorithmic}[1]

\Require The domain $\mathcal{D}$, the set of transit candidates $\mathscr{T}$, the source $s$, the goal $g$, and the privacy parameters ($k$, $\ell$)
\Ensure The best partition $\Psi^{*}$ of $\mathscr{T}$

\State $|ap|^{*} \leftarrow 0$, $mac^{*} \leftarrow \infty$, $\Psi^{*} = \emptyset$    

\Function{Depth\_BB\_Search}{$\Psi$, $\mathscr{T}^{U}$}
\If{$\mathscr{T}^{U}$ is empty}
    \parState {$\Psi_{+} \leftarrow \{\psi | \psi \in \Psi$ such that $\psi$ satisfies all conditions of Theorem~1 $\}$ }
    \State $|ap| \leftarrow \sum_{\psi \in \Psi_{+}} |\psi|$
    \State $mac \leftarrow \sum_{\psi \in \Psi_{+}} ac(\psi) / |ap|$

    \If{($|ap| > |ap|^{*}$) or ($|ap| = |ap|^{*}$ and $mac < mac^{*}$)}
        \State $|ap|^{*} \leftarrow |ap|$, $mac^{*} \leftarrow mac$, $\Psi^{*} \leftarrow \Psi$
    \EndIf
    \State
    \Return
\EndIf
\State $n \leftarrow$ top of $\mathscr{T}^{U}$
\For{$\psi \in \Psi$}
    \If{Prunable($\psi$, $\{n\}$)}
        Continue
    \EndIf
    \State Depth\_BB\_Search($(\Psi \setminus \psi) \cup (\psi \cup n)$, $\mathscr{T}^{U}$ $\setminus n$)
\EndFor
\State Depth\_BB\_Search($\Psi \cup \{n\}$, $\mathscr{T}^{U}$ $\setminus n$)
\EndFunction

\State
\State Depth\_BB\_Search($\emptyset$, $\mathscr{T}$)
\State
\Return $\Psi^{*}_{+} \cup \{\bigcup \Psi^{*} \setminus \Psi^{*}_{+} \}$

\end{algorithmic}
\end{algorithm}

\section{Additional Results}
\label{supp:additional-results}

Fig.~\ref{den201d} to \ref{orz201d} show the qualitative examples of $(2, 1, \infty)$-Anonymized Paths on \textit{den201d}, \textit{lak202d}, \textit{lak510d}, \textit{orz000d}, and \textit{orz201d}.

Figure \ref{fig:den101} illustrates a comparison between the $(2, 1, \infty)$-Anonymized Path and the $(2, 10, \infty)$-Anonymized Path applied to the same problem. We can see that in the $(2, 1, \infty)$-Anonymized Path, the two nodes in the left corner are grouped into the same partition. However, in the $(2, 10, \infty)$-Anonymized Path, they belong to different partitions.

Tab.~\ref{tab:impact-m} presents the numeric values shown in Fig.~6.

\begin{figure}[H]
    \centering
    \includegraphics[width=\linewidth]{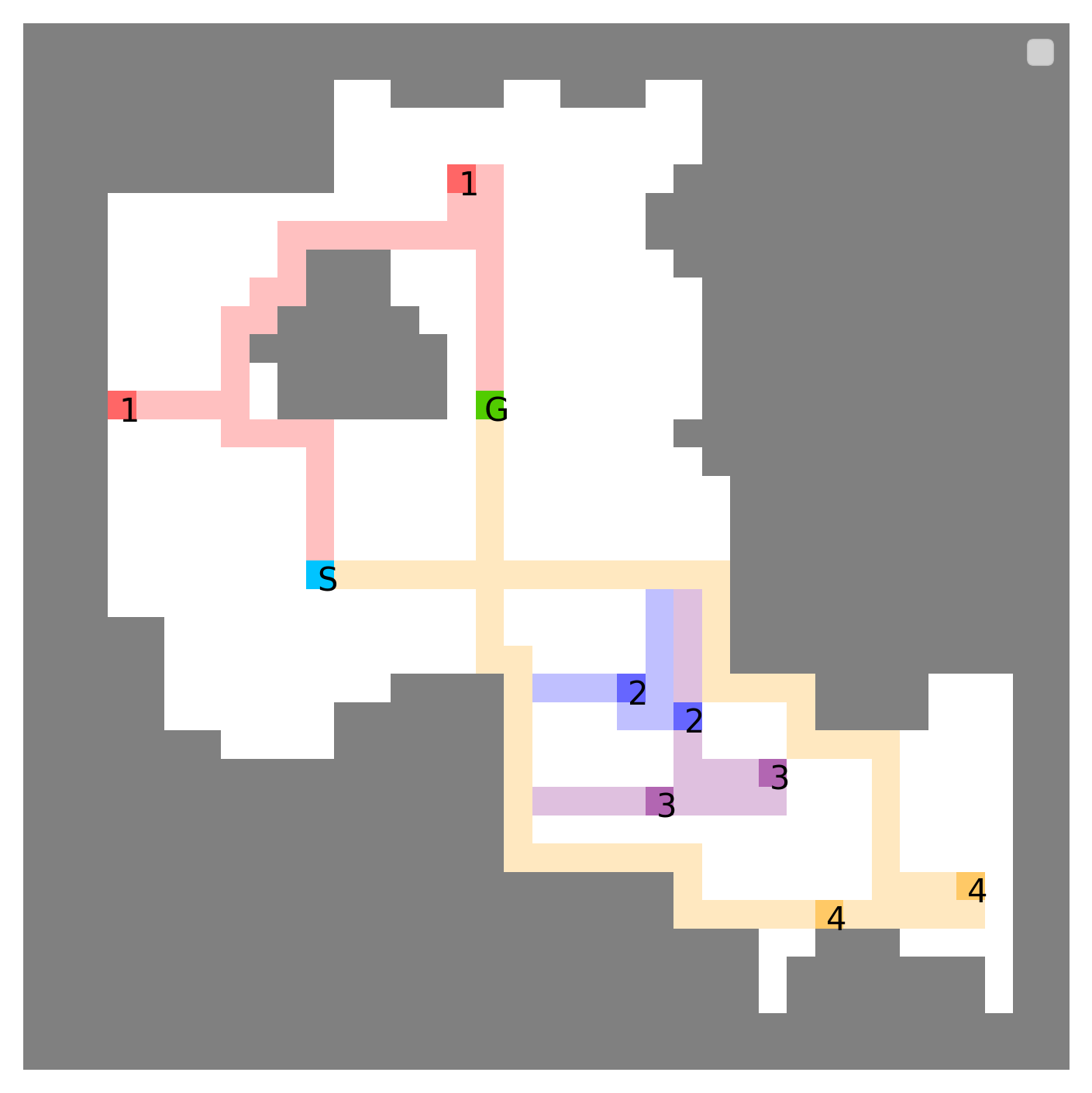}
    \caption{Optimal $(2, 1, \infty)$-Anonymized Path on \textit{den201d} found by Merge-BB with CostAsc and $h_{tunnel}$.}
    \label{den201d}
\end{figure}

\begin{figure}[H]
    \centering
    \includegraphics[width=\linewidth]{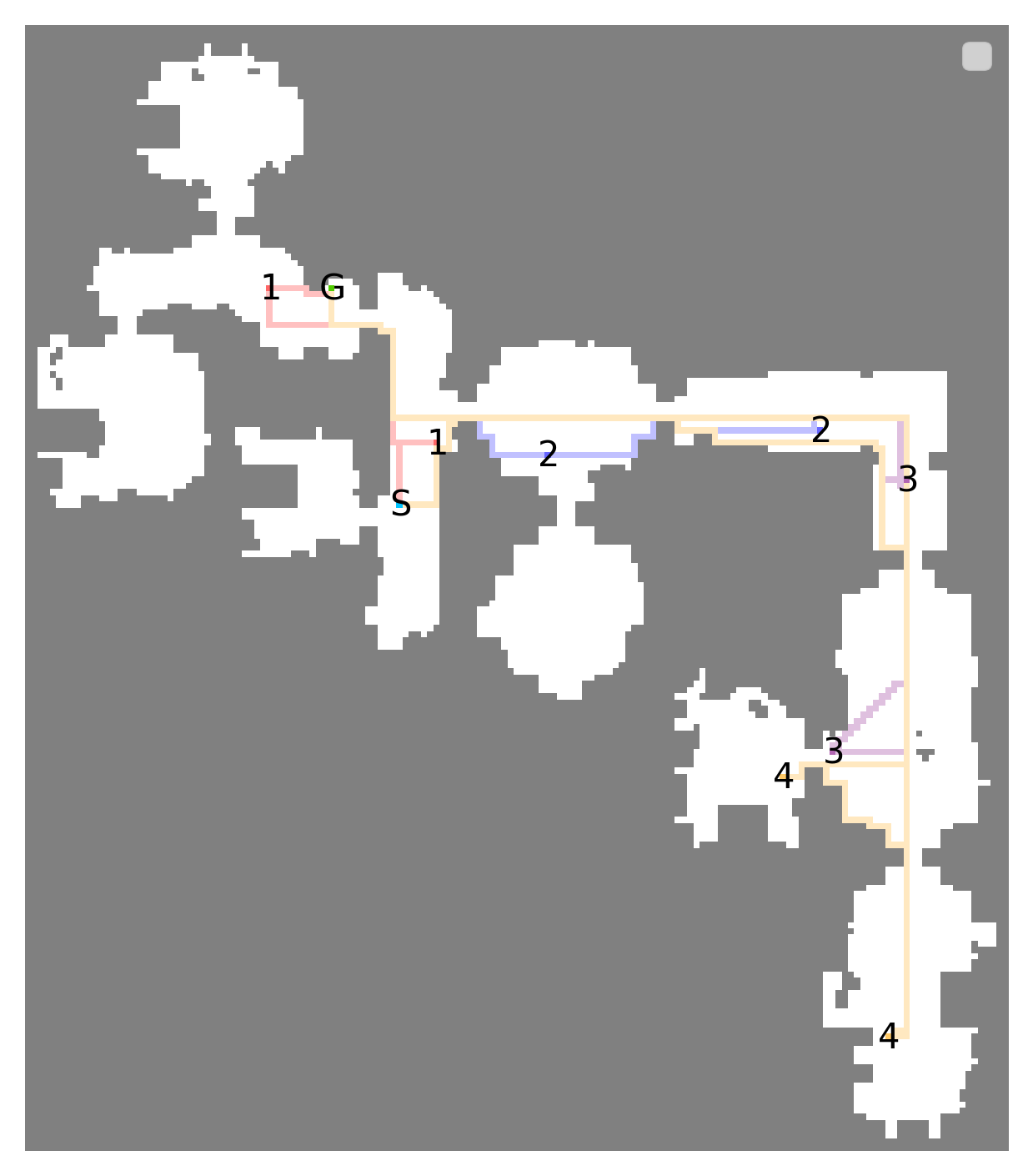}
    \caption{Optimal $(2, 1, \infty)$-Anonymized Path on \textit{lak202d} found by Merge-BB with CostAsc and $h_{tunnel}$.}
    \label{lak202d}
\end{figure}

\begin{figure}[H]
    \centering
    \includegraphics[width=\linewidth]{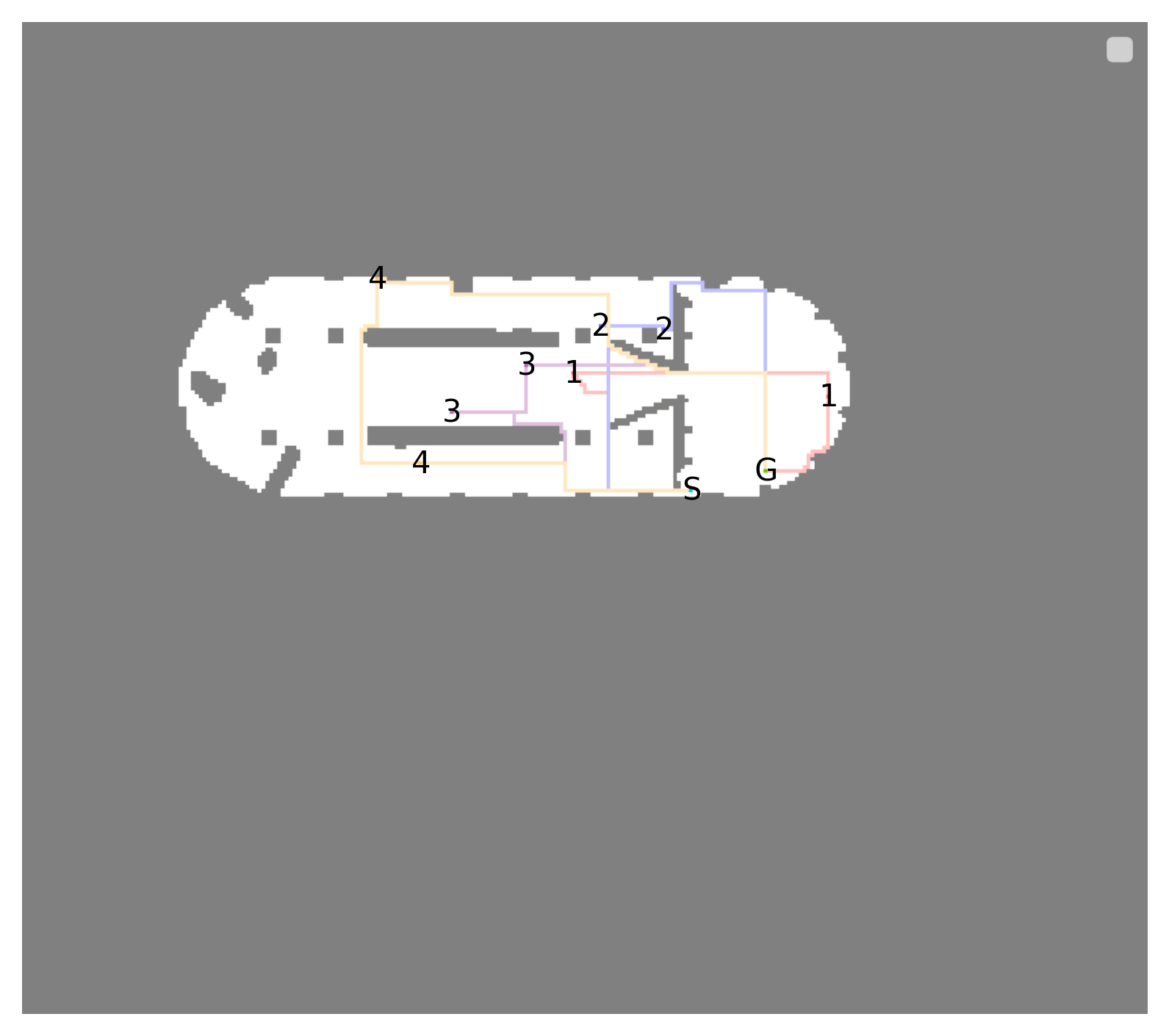}
    \caption{Optimal $(2, 1, \infty)$-Anonymized Path on \textit{lak510d} found by Merge-BB with CostAsc and $h_{tunnel}$.}
    \label{lak510d}
\end{figure}

\begin{figure}[H]
    \centering
    \includegraphics[width=\linewidth]{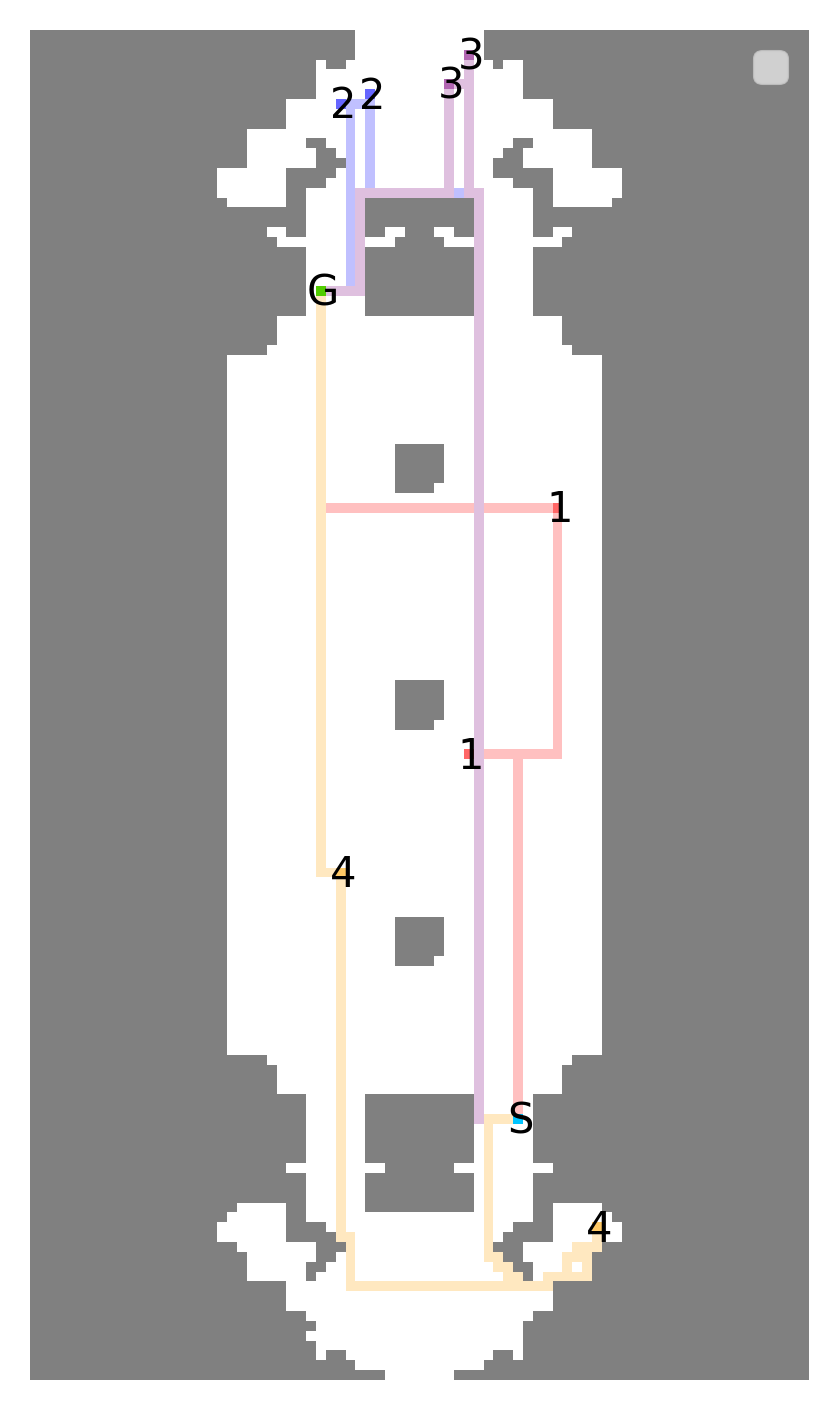}
    \caption{Optimal $(2, 1, \infty)$-Anonymized Path on \textit{orz000d} found by Merge-BB with CostAsc and $h_{tunnel}$.}
    \label{orz000d}
\end{figure}

\begin{figure}[H]
    \centering
    \includegraphics[width=\linewidth]{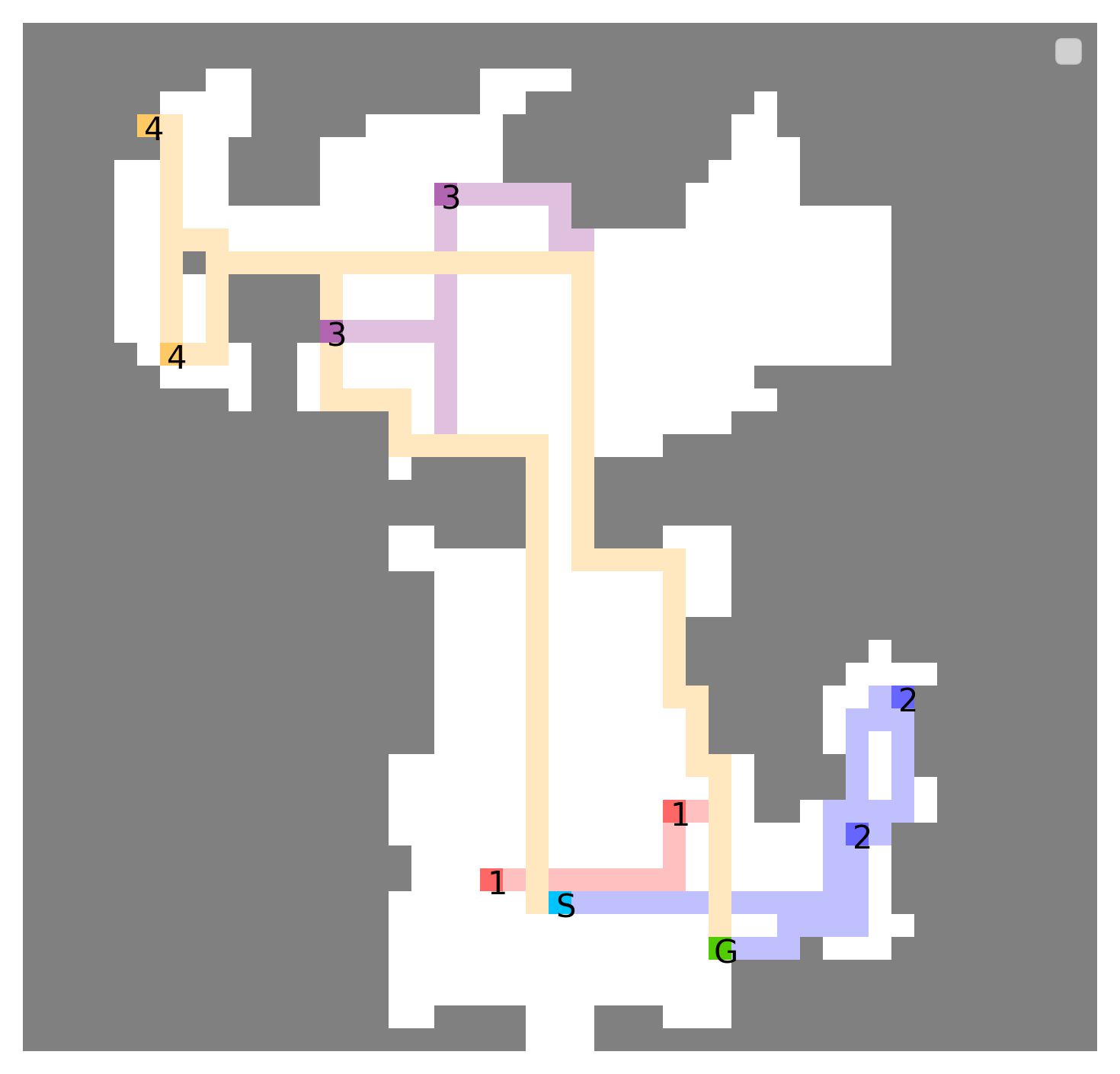}
    \caption{Optimal $(2, 1, \infty)$-Anonymized Path on \textit{orz201d} found by Merge-BB with CostAsc and $h_{tunnel}$.}
    \label{orz201d}
\end{figure}

\begin{figure}[H]
\centering
  \begin{minipage}[b]{\linewidth}
    \centering
    \includegraphics[width=\linewidth]{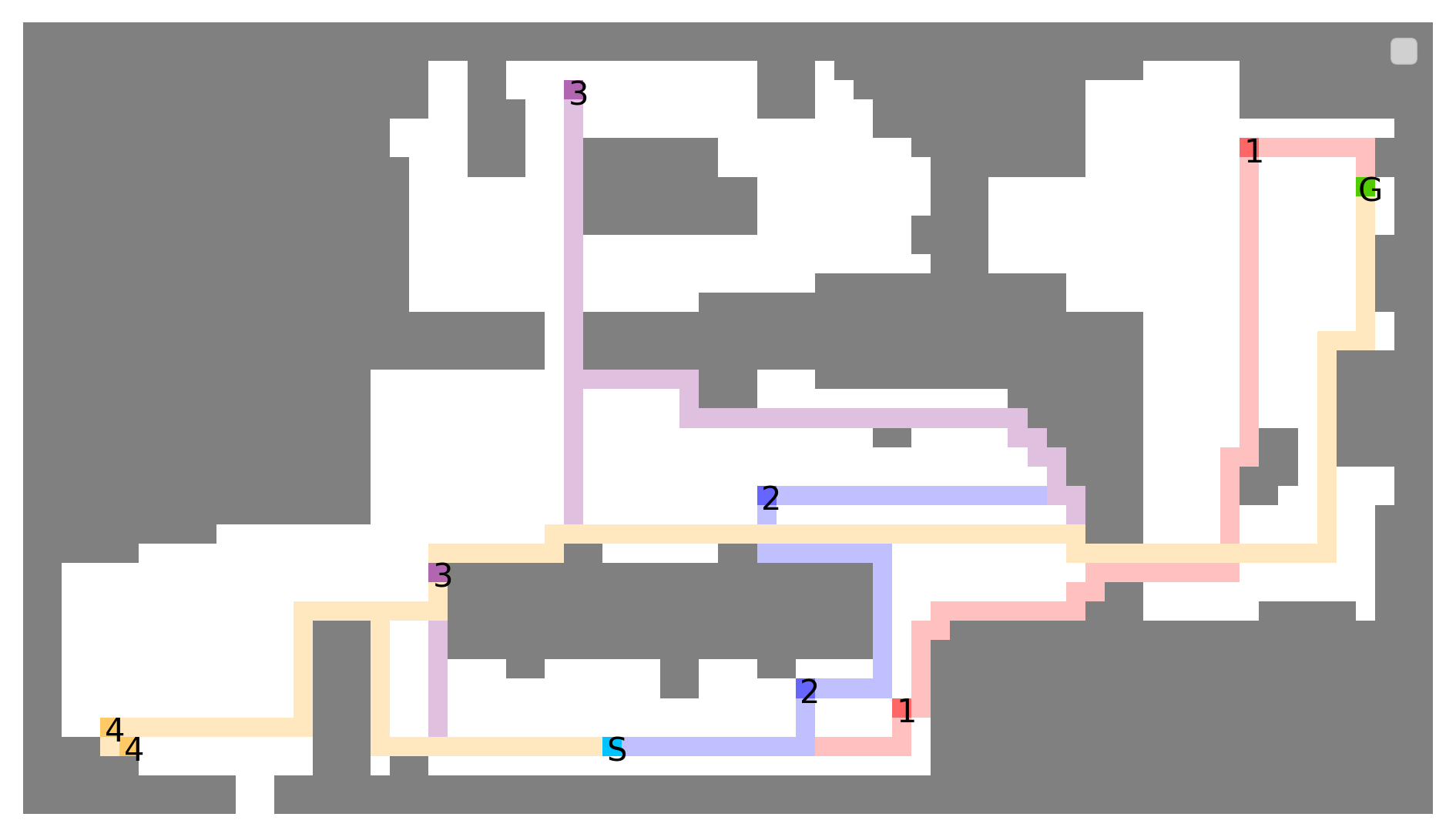}
    \subcaption{$(2, 1, \infty)$-Anonymized Path}
  \end{minipage}
  \\
  \begin{minipage}[b]{\linewidth}
    \centering
    \includegraphics[width=\linewidth]{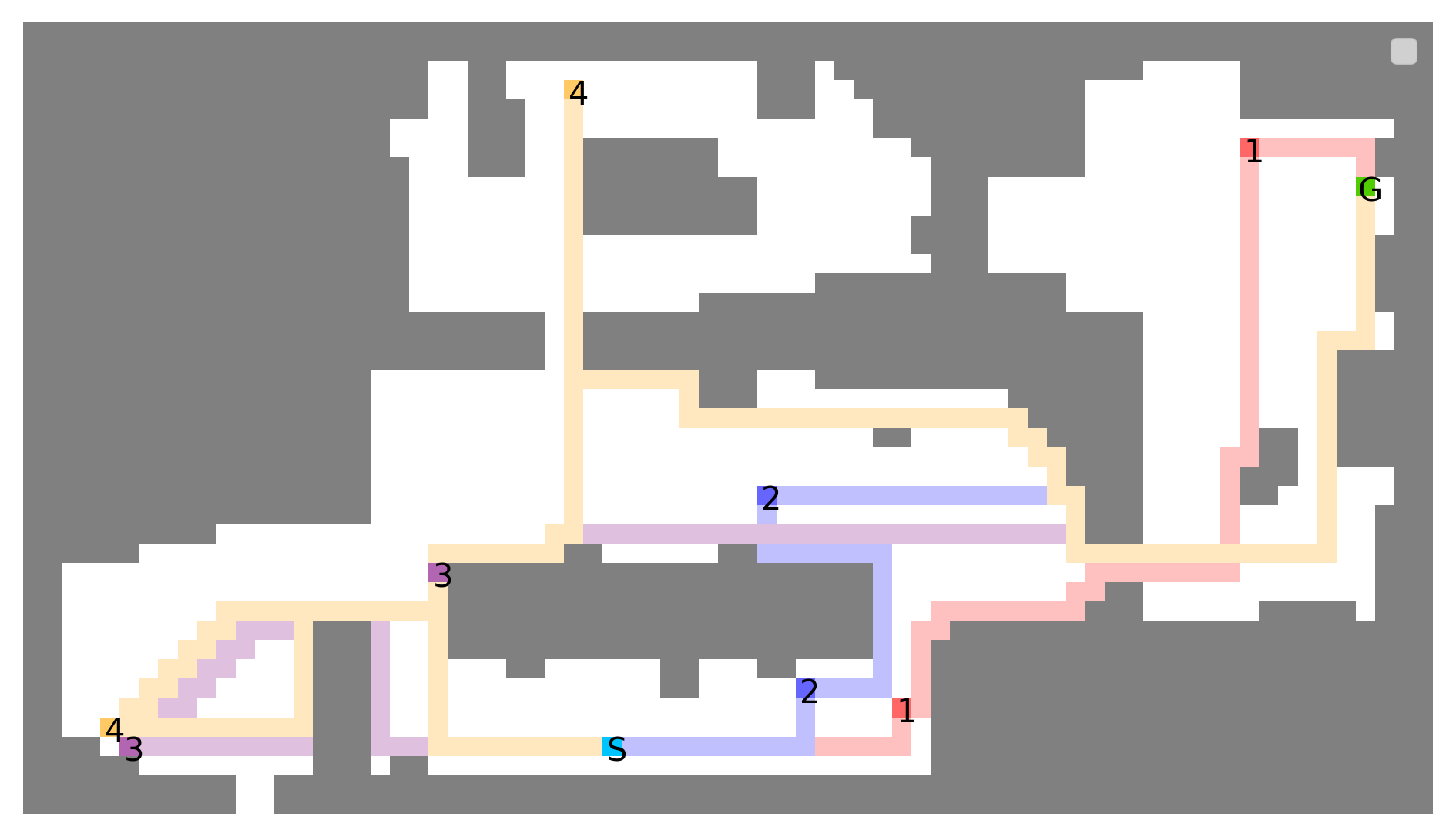}
    \subcaption{$(2, 10, \infty)$-Anonymized Path}
  \end{minipage}
  \caption{Optimal Anonymized Paths on \textit{den101d} found by Merge-BB with CostAsc and $h_{tunnel}$.}
  \label{fig:den101}
\end{figure}

\begin{table}[H]
\begin{tabular}{@{}c|c|cccccc@{}}
\toprule
  & $m / |\pi^{*}|$ & 0.1             & 0.3             & 0.5            & 1.0   & 5.0  & 10.0 \\ \midrule
k & Planner         &                 &                 &                &       &      &      \\ \midrule
\multirow{3}{*}{2} & $m$-Pbp & \textbf{0.00707} & \textbf{0.0313} & \textbf{0.0803} & \textbf{0.188} & \textbf{0.216} & \textbf{0.229} \\
  & Cbp             & 0.0116          & 0.0671          & 0.189          & 0.692 & 4.71 & 9.72 \\
  & Rbp             & 0.47            & 1.63            & 2.81           & 5.73  & 29   & 58.1 \\ \midrule
\multirow{3}{*}{3} & $m$-Pbp & 0.0228           & 0.1             & \textbf{0.192}  & \textbf{0.416} & \textbf{0.628} & \textbf{0.644} \\
  & Cbp             & \textbf{0.0149} & \textbf{0.0846} & 0.225          & 0.741 & 4.76 & 9.76 \\
  & Rbp             & 0.47            & 1.63            & 2.81           & 5.73  & 29   & 58.1 \\ \midrule
\multirow{3}{*}{5} & $m$-Pbp & 0.0551           & 0.2             & 0.365           & \textbf{0.793} & \textbf{1.64}  & \textbf{1.87}  \\
  & Cbp             & \textbf{0.0218} & \textbf{0.133}  & \textbf{0.316} & 0.864 & 4.93 & 9.95 \\
  & Rbp             & 0.47            & 1.63            & 2.81           & 5.73  & 29   & 58.1 \\ \bottomrule
\end{tabular}
    \caption{Impact of $k$ and $m$ on MAC ($\ell=1$ and $|\mathscr{T}|=8$). See Fig.6 for the visualization. Note that the values of Rbp are the same regardless of $k$ since Rbp does not depend on $k$.}
    \label{tab:impact-m}
\end{table}

\makeatletter\@input{mainaux.tex}\makeatother